\documentclass{article} 
\usepackage{iclr2027_conference,times}

\usepackage{amsmath,amsfonts,bm}

\def\eqref#1{equation~\ref{#1}}

\def\1{\bm{1}}

\DeclareMathAlphabet{\mathsfit}{\encodingdefault}{\sfdefault}{m}{sl}
\SetMathAlphabet{\mathsfit}{bold}{\encodingdefault}{\sfdefault}{bx}{n}

\usepackage{hyperref}
\usepackage{url}
\usepackage{amsmath,amssymb,amsthm}
\usepackage{graphicx}
\usepackage{booktabs}
\usepackage[table]{xcolor}

\definecolor{cPos}{RGB}{58,164,110}
\definecolor{cNeg}{RGB}{217,106,89}
\definecolor{stdgray}{gray}{0.45}

\newtheorem{definition}{Definition}
\newtheorem{proposition}{Proposition}

\title{Stale, Misattributed, or Late: Where Personal Memory Fails Before Generation}

\author{
\textbf{Haonan Deng}\textsuperscript{1}
\quad
\textbf{Wichinpong Park Sinchaisri}\textsuperscript{2}
\quad
\\[0.6em]
\textsuperscript{1}
University of California, Berkeley
\\
\textsuperscript{2}
Haas School of Business, University of California, Berkeley
}

\iclrfinalcopy 

\begin{document}

\maketitle

\begin{abstract}
Personal memory for language agents is usually judged by whether the final answer is correct. That score hides errors that arise before generation: the memory block may contain an obsolete value, a fact about the wrong person, or no useful fact before the serving deadline. We measure these failures directly. Using Personal Fact Memory (PFM) as a reference layer, we find that temporal validity is primarily a property of memory construction in our setting. On a controlled revision benchmark, serving only the active value of each correctly keyed slot eliminates observed stale exposure; without update resolution, 70.3\% of prompts expose a superseded value. Once retrievers share the same active store and participant information, participant-aware BM25 is equivalent to the reference ranker within a prespecified $\pm0.02$ margin. The harder problem is assigning revisions to the right slot. Missed merges leave stale values active, whereas false merges silently remove current values; four LLM key assigners achieve higher key recall than a rule extractor yet produce lower clean-retrieval rates, and open-domain merge recall on LongMemEval never exceeds 0.062. Misattribution survives validity filtering: an entity posterior reduces same-name exposure on controlled data but cannot distinguish identically named speakers in LoCoMo. Two frozen language models reproduce prompt errors in generated text. Retrieval latency varies across rankers, but prompt prefill dominates turn-level latency on our hardware. These results argue for evaluating agent memory before generation, separating stored-state validity, identity resolution, abstention, and serving latency.
\end{abstract}

\section{Introduction}

A user drafts a message to a colleague, Alex Chen, and the assistant offers a continuation of ``the standup with Alex is at''. A useful continuation contains the time the two of them most recently agreed on. Before the model writes a word, the memory layer can fail in three different ways: it can surface 9:30am after the meeting moved to 10:45am, retrieve the standup time for a different Alex, or finish after prefill has begun. We call these failures staleness, misattribution, and lateness.

Several memory systems already handle temporal revisions at write time by linking a fact to its subject and property, closing the old value when a new one arrives, and retrieving only the value that remains active~\citep{rasmussen2025zep,chhikara2025mem0,yadav2026temporal,wang2026toki}. The same transaction-time principle is standard in temporal databases~\citep{jensen1999temporal}. The rule itself is not our contribution. The unresolved questions are what it buys when its prerequisites are imperfect and what errors remain after it works. Existing evaluations usually score the final answer, including LoCoMo, LongMemEval, MemoryAgentBench, and Memora~\citep{maharana2024evaluating,wu2024longmemeval,hu2026evaluating,uddin2026recall}. A final-answer score cannot distinguish a clean prompt from one that contains both the current value and a superseded value if the model happens to choose correctly. We therefore measure the memory block before generation: whether it contains the current value, a superseded value, a same-named other person's value, or no useful value before the serving deadline.

The first result is that temporal validity is largely determined by the stored state. With correct keys, restricting retrieval to active facts removes observed stale exposure for every retriever we test. Once the active store and participant information are held fixed, participant-aware BM25 is equivalent to PFM within a prespecified $\pm0.02$ margin, and dense retrieval does not improve clean retrieval. The harder problem is assigning revisions the right key. Missed merges leave obsolete values active, whereas false merges remove current values from the active store. Four LLM key assigners achieve key recall above 0.980 on paraphrased benchmark text yet yield lower clean-retrieval rates than a lower-recall rule extractor because they occasionally merge across slots. On LongMemEval, where the slot vocabulary is not supplied, merge recall remains at or below 0.062 even when existing store keys are exposed to the assigner.

The remaining failures require different information. Same-name competitors remain in active memory, and a fixed-$k$ retriever fills the prompt even when the store has no answer. A lightweight entity posterior reduces wrong-person exposure on controlled data while preserving facts about third parties, but it cannot distinguish two speakers represented by the same surface form in LoCoMo. A schema-aware slot gate abstains reliably on the controlled benchmark; a schema-free margin gate does not transfer. Serving time is separate again. Dense query encoding is slower than lexical retrieval, but on our hardware the memory block itself adds far more latency through model prefill. Experiments with two frozen language models confirm that prompt-level errors can propagate into generated text.

We make three contributions:
\begin{enumerate}
\item \textbf{Prompt-level evaluation.} We define current-value recall, stale exposure, misattribution, abstention, and clean retrieval by deadline, including requests for which the store contains no valid answer.
\item \textbf{Validity under imperfect updates.} We measure keyed supersession under controlled key noise, key granularity, out-of-order arrival, four LLM key assigners on paraphrased text, and open-domain key discovery on LongMemEval. The experiments separate the store invariant from the accuracy of the extractor that supplies its keys.
\item \textbf{Identity, abstention, and serving cost.} We evaluate identity resolution and abstention on controlled and public corpora, compare lexical, dense, hybrid, and query-time model baselines, and trace prompt-level errors into two frozen language models while measuring retrieval and prefill latency.
\end{enumerate}

\section{Related Work}

\paragraph{Write-time temporal validity.}\label{sec:related_validity} Agent memories rank observations by recency and relevance~\citep{park2023generative}, add forgetting curves~\citep{zhong2024memorybank}, page between memory tiers~\citep{packer2023memgpt}, link atomic notes~\citep{xu2026mem}, or retrieve over entity graphs and abstractions~\citep{gutierrez2024hipporag,xia2026memora}. A subset maintains one current value at write time, either through a model that judges contradiction~\citep{rasmussen2025zep,chhikara2025mem0} or through deterministic updates over structured facts~\citep{yadav2026temporal,wang2026toki}; other systems resolve conflicts at read time~\citep{banerjee2026apex} or order facts by stated rather than dialogue time~\citep{su2026beyond}. PFM uses the same basic transaction-time rule. Proposition~\ref{prop:one_active} states its store invariant and does not constitute a new temporal-memory mechanism. Our focus is the premise these systems require: whether a revision is recognized, assigned to the right slot, and ordered correctly. MemStrata reports that temporal validity sharply reduces stale-fact errors, and TOKI provides guarantees conditional on its structured representation; Sections~\ref{sec:keynoise} and~\ref{sec:real} measure what happens when the representation needed by that update rule is imperfect or unavailable.

\paragraph{Prompt-level memory evaluation and serving cost.} LoCoMo and LongMemEval evaluate long-horizon memory through final answers~\citep{maharana2024evaluating,wu2024longmemeval}; MemoryAgentBench includes selective forgetting as a memory competency~\citep{hu2026evaluating}; Memora penalizes reliance on invalidated memory~\citep{uddin2026recall}; MemConflict and STALE target conflicting or expired memories~\citep{tao2026memconflict,chao2026stale}. MemTrace is closest in separating retrieved context from the final answer~\citep{long2026memtrace}. We instead label the prompt by failure mode. A block can contain the current value, a superseded value, and another person's value simultaneously, so final-answer accuracy need not reveal the state presented to the model. We also treat serving time as part of the outcome. Recent systems study memory cost, vector retrieval on constrained hardware, and when dense retrieval is worth its expense~\citep{zhao2025ame,yuan2026agentir,zhou2026we}; read-copy-update and near-real-time search motivate the immutable publication path used here~\citep{mckenney1998read,gospodnetic2010lucene}. Our measurements place these costs beside prompt correctness because a clean block that arrives after prefill has begun cannot affect the current response.

\section{Problem Formulation}
\label{sec:problem}

Requests are indexed by $t$ and arrive at wall-clock time $\tau_t$. Let $\mathcal{M}_t$ denote the facts retained by the memory layer at that time, including facts preserved only as history. Request $t$ provides local context $c_t$, participant set $P_t$, fact limit $k$, character budget $B$, and deadline $D_t$ by which the memory block must be ready for prefill.

\begin{definition}[Personal fact]
A stored fact is $f=(z_f,E_f,P_f,\kappa_f,a_f,b_f,s_f)$, where $z_f$ is prompt-ready text, $E_f$ a set of normalized entities, $P_f$ the participants of the source observation, $\kappa_f$ an optional key naming a mutable slot, $[a_f,b_f)$ the interval during which the store treats $f$ as active ($b_f=\infty$ until supersession), and $s_f$ the time the fact was last observed. The active view at request $t$ is $\mathcal{A}_t=\{f\in\mathcal{M}_t:a_f\le\tau_t<b_f\}$.
\end{definition}

The activity interval records the store's state, not the truth time in the world. If an extractor misses a revision, an outdated value can remain active. A retrieval policy $\pi$ maps $(\mathcal{M}_t,c_t,P_t)$ to an ordered set $S_{\pi,t}\subseteq\mathcal{M}_t$ in time $L_{\pi,t}$, subject to $|S_{\pi,t}|\le k$ and a serialized length of at most $B$. A policy may restrict itself to $\mathcal{A}_t$ or ignore activity intervals, so validity-agnostic retrievers remain admissible comparisons.

For evaluation, $R_t(f)=1$ when $f$ states a value for the slot requested at $t$, for any entity whose name matches the request surface form; $I_t(f)=1$ when $f$ concerns the intended entity; and $V_t(f)=1$ when its value is still current. The target set is $\mathcal{G}_t=\{f:R_t(f)=I_t(f)=V_t(f)=1\}$, and the policy observes none of these labels. Current-value recall is $C_{\pi,t}=\mathbf{1}\{S_{\pi,t}\cap\mathcal{G}_t\neq\varnothing\}$, stale exposure is $E^{\mathrm{stale}}_{\pi,t}=\mathbf{1}\{\exists f\in S_{\pi,t}:R_t(f)=I_t(f)=1,V_t(f)=0\}$, and misattribution is $E^{\mathrm{id}}_{\pi,t}=\mathbf{1}\{\exists f\in S_{\pi,t}:R_t(f)=1,I_t(f)=0\}$. When $\mathcal{G}_t=\varnothing$, recall is zero for every policy, including one that correctly returns nothing, so clean retrieval requires a separate unanswerable case:
\begin{equation}
Y_{\pi,t}=
\begin{cases}
C_{\pi,t}\,(1-E^{\mathrm{stale}}_{\pi,t})(1-E^{\mathrm{id}}_{\pi,t}), & \mathcal{G}_t\neq\varnothing,\\[2pt]
1-\mathbf{1}\{\exists f\in S_{\pi,t}:R_t(f)=1\}, & \mathcal{G}_t=\varnothing,
\end{cases}
\qquad Y_{\pi,t}(D)=Y_{\pi,t}\,\mathbf{1}\{L_{\pi,t}\le D\},
\label{eq:clean}
\end{equation}
We report the answerable and unanswerable cases separately because frequent abstention can improve the latter while harming the former. We also record $A_{\pi,t}=\mathbf{1}\{S_{\pi,t}=\varnothing\}$. The exposure indicators are not mutually exclusive: a prompt may contain both the current and a superseded value, which recall alone does not reveal.

\paragraph{Keyed updates.} Suppose every recognized revision of a slot carries the same key $\kappa$ and updates are applied in receipt order. Committing a new fact $f'$ with key $\kappa$ at time $u$ sets $a_{f'}=u$ and $b_{f'}=\infty$ and closes the active predecessor by setting $b_f=u$; the predecessor remains as history. The commit is atomic with respect to readers.

\begin{proposition}[At most one active value per key]
\label{prop:one_active}
Fix a key $\kappa$. If updates carrying $\kappa$ are committed serially and atomically by the rule above, and at most one fact with key $\kappa$ is active initially, then every reader-visible state satisfies $|\{f\in\mathcal{A}:\kappa_f=\kappa\}|\le 1$.
\end{proposition}

Appendix~\ref{app:proofs} gives the proof. The proposition is a state invariant, not a correctness guarantee. Under out-of-order arrival, the last value to arrive remains active rather than the value stated latest in event time. A missed merge assigns two revisions different keys and leaves both active; a false merge assigns a revision to the wrong slot key and can remove a current value. Sections~\ref{sec:keynoise} and~\ref{sec:real} measure both failures.

\section{A Reference Memory Layer}
\label{sec:pfm}

We use Personal Fact Memory (PFM) as a reference layer because its construction and serving decisions are explicit and can be held fixed across multiple comparisons. PFM implements the keyed-supersession rule above, with the update decision made by key comparison at commit time rather than by a model call. We use this established mechanism as the state-management component of the reference layer; Appendix~\ref{app:config} specifies the scoring rule and fixed parameters used in the experiments.

An extractor converts each observation into fact drafts with text, entities, and an optional slot key. On the controlled benchmark, a rule assigns a key when the sentence names a gazetteer entity, a slot noun, and a value of that slot's type. The public corpora in Section~\ref{sec:real} provide no such slot inventory, so their facts begin as sentences and receive keys from a model when a keyed arm is evaluated. A keyed draft that repeats the active value is a re-observation; a different value supersedes it. A keyless draft is absorbed only when an active fact sharing an entity has token Jaccard similarity of at least $0.8$ and the same value-bearing tokens. Superseded facts leave the retrieval index but remain as history, and pruning exempts the active fact of every key. Serving ranks active facts with BM25F~\citep{robertson2004simple,robertson2009probabilistic} over text, entity, and participant fields, together with recency, usage, and a one-hop association term, then takes facts in score order until $k$ or $B$ binds. No embedding model or language model runs on this path. A background worker can instead publish an immutable snapshot for request-time reads; Appendix~\ref{app:serving} states the conditional waiting bound and measures the handoff and scheduling terms it requires.

\section{Identity Resolution and Abstention}
\label{sec:identity}

Supersession determines which stored value of a keyed slot is active. It does not determine which entity an ambiguous request refers to or whether the store contains an answer. We treat those as serving decisions over the retrieved candidate set.

A request token is an ambiguous \emph{mention} when it matches more than one stored entity, as a first name can. For the candidate set $\mathcal{E}_t$, we score each entity and normalize the scores into a posterior:
\begin{equation}
\psi_t(e)=w_p\,\mathbf{1}\{e\in P_t\}+w_g\!\!\sum_{p\in P_t}\!\frac{w_{pe}(\tau_t)}{\sum_v w_{pv}(\tau_t)}+w_r\,2^{-(\tau_t-s_e)/T_{\mathrm{rec}}},
\qquad
\pi_t(e)=\frac{e^{\psi_t(e)/T}}{\sum_{e'\in\mathcal{E}_t}e^{\psi_t(e')/T}},
\label{eq:identity}
\end{equation}
where $s_e$ is the most recent active fact about $e$, $w=(1,1,0.25)$, and $T=0.5$. The three terms favor an entity that is the current interlocutor, is associated with the interlocutor in the memory graph, or has been observed recently. Retrieval retains facts about $\arg\max_e\pi_t(e)$ and facts that refer to none of the candidates. A hard participant filter is the limiting case that keeps only the interlocutor; it removes ambiguity when facts concern that person but can discard facts about third parties.

Abstention requires a separate decision. The slot gate identifies the requested slot from its noun and trailing cue and returns an empty block unless a retrieved fact states that slot for an entity with $\pi_t(e)\ge\tau$. The \emph{margin gate} requires no slot inventory and returns an empty block unless the top fact scores at least $\tau$ times the shortlist mean and shares one third of the request's content terms. The slot gate is more informative but depends on a known schema; Appendix~\ref{app:identity} evaluates both under answerable and unanswerable requests.

\section{Experimental Setup}
\label{sec:setup}

\paragraph{Corpora.} The controlled benchmark contains 160 revision chains per instance over person or project slots, with zero to four revisions per chain and 1{,}000 filler messages. Queries are unfinished message prefixes and carry the chain's participant as $P_t$. People are named by first name in the request, so chains on the same slot for people sharing a first name create wrong-person competitors; 40 to 50 of the 240 queries per instance are ambiguous. Five independently seeded instances yield 1{,}200 queries. Values are unique within an instance, which makes prompt exposure measurable by whole-token containment. We also use two public corpora that remove structure supplied by the generator. LongMemEval~\citep{wu2024longmemeval} provides knowledge-update questions with earlier and later answer-bearing turns, allowing us to study whether open-domain key assignment causes supersession to fire. LoCoMo~\citep{maharana2024evaluating} contains multiple conversations with speakers sharing the same first name, which lets us measure wrong-person exposure by provenance; per-conversation stores provide unanswerable requests. Appendices~\ref{app:bench} and~\ref{app:real} give the construction and filtering details.

\paragraph{Arms and inference.} Every retrieval arm uses the same store, greedy selection rule, and memory-block format with $k=6$ and $B=800$ ($B=1200$ on the public corpora). BM25 uses text only; ``+ validity'' restricts retrieval to active facts; ``+$P_t$'' supplies participant names to both the query and fact text; Dense uses all-MiniLM-L6-v2; and Hybrid fuses BM25 and Dense shortlists by reciprocal rank. Queries from the same chain share labels, and chains within an instance share a benchmark seed. Confidence intervals therefore use a two-level bootstrap over seeds and chains; rates of zero or one use Wilson intervals over chains. Paired comparisons use a chain-level permutation test. When we claim equivalence, we use two one-sided tests against a clean-retrieval margin fixed in advance at $\pm0.02$~\citep{schuirmann1987comparison} rather than treating a nonsignificant difference as evidence of equivalence.

\section{Results}

\subsection{Supersession and Retrieval Ranking}
\label{sec:temporal}

\begin{table}[t]
\centering
\small
\setlength{\tabcolsep}{4pt}
\begin{tabular}{@{}llrrrrr@{}}
\toprule
Store & $B$ & Current $\uparrow$ & Stale $\downarrow$ & Wrong-person $\downarrow$ & Co-injection $\downarrow$ & Clean $\uparrow$ \\
\midrule
Keyed & 800 & 0.998 {\scriptsize(0.993, 1.000)} & \cellcolor{cPos!60}0.000 {\scriptsize(0.000, 0.005)} & 0.121 {\scriptsize(0.094, 0.149)} & 0.000 {\scriptsize(0.000, 0.005)} & \cellcolor{cPos!60}0.877 {\scriptsize(0.850, 0.904)} \\
Keyless & 800 & 0.950 {\scriptsize(0.923, 0.971)} & 0.703 {\scriptsize(0.654, 0.749)} & 0.004 {\scriptsize(0.000, 0.012)} & 0.701 {\scriptsize(0.652, 0.746)} & 0.247 {\scriptsize(0.210, 0.284)} \\
\midrule
Keyed & 240 & 0.869 {\scriptsize(0.836, 0.903)} & 0.000 {\scriptsize(0.000, 0.005)} & 0.000 {\scriptsize(0.000, 0.005)} & 0.000 {\scriptsize(0.000, 0.005)} & 0.869 {\scriptsize(0.836, 0.903)} \\
Keyless & 240 & 0.823 {\scriptsize(0.791, 0.852)} & 0.446 {\scriptsize(0.396, 0.501)} & 0.000 {\scriptsize(0.000, 0.005)} & 0.438 {\scriptsize(0.388, 0.493)} & 0.385 {\scriptsize(0.331, 0.435)} \\
\bottomrule
\end{tabular}
\caption{Prompt-level effects of keyed supersession. Each entry is the share of 1{,}200 queries whose memory block contains the current value, a superseded value of the same slot (stale), a same-named wrong person's value, both current and stale values (co-injection), or the current value with neither error (clean). $B$ is the character budget. Parentheses report 95\% intervals; for zero observed stale exposure, the Wilson upper bound over 800 chains is 0.005. Shading marks the focal keyed comparisons.}
\label{tab:temporal}
\end{table}

\begin{table}[t]
\centering
\small
\setlength{\tabcolsep}{4pt}
\begin{tabular}{@{}lrrrrrr@{}}
\toprule
Method & Current $\uparrow$ & Stale $\downarrow$ & Wrong-person $\downarrow$ & Clean $\uparrow$ & p50 (ms) $\downarrow$ & p99 (ms) $\downarrow$ \\
\midrule
PFM & 0.998 & 0.000 & 0.121 & \cellcolor{cPos!60}0.877 {\scriptsize(0.850, 0.904)} & 0.25 & 0.69 \\
BM25 + validity & 0.891 & 0.000 & 0.120 & 0.774 {\scriptsize(0.737, 0.809)} & 0.09 & 0.26 \\
BM25 + validity + $P_t$ & 1.000 & 0.000 & 0.128 & \cellcolor{cPos!60}0.873 {\scriptsize(0.847, 0.896)} & 0.10 & 0.29 \\
Dense + validity + $P_t$ & 1.000 & 0.000 & 0.164 & 0.836 {\scriptsize(0.803, 0.868)} & 18.3 & 93.0 \\
Hybrid + validity + $P_t$ & 1.000 & 0.000 & 0.157 & 0.843 {\scriptsize(0.813, 0.872)} & 18.2 & 82.4 \\
\midrule
BM25 + $P_t$ (no validity) & 0.997 & 0.802 & 0.007 & 0.195 {\scriptsize(0.164, 0.227)} & 0.18 & 0.46 \\
\bottomrule
\end{tabular}
\caption{Retrieval arms under the same stores, selection rule, and prompt budget ($k{=}6$, $B{=}800$). Columns follow Table~\ref{tab:temporal}. ``+ validity'' restricts retrieval to active facts; ``+$P_t$'' supplies participant names to the query and fact text. The final row is the strongest validity-agnostic arm; Appendix~\ref{app:full} reports all arms. Parentheses report 95\% intervals for clean retrieval. Shading marks the equivalence comparison.}
\label{tab:baselines}
\end{table}

The keyed and keyless stores receive identical fact drafts; the keyless store simply ignores their keys. None of the 1{,}200 keyed prompts contains a superseded value, which bounds stale exposure at 0.005 rather than establishing a population rate of zero. In the keyless store, 70.3\% of prompts expose a stale value, usually beside the current one. The benchmark deliberately revises 80\% of its chains. Reweighting to a store in which 70\% of facts are never revised leaves keyed PFM at 0.877 clean retrieval and raises the keyless store to 0.603, reducing the gap without changing its direction. The keyed store's remaining errors are wrong-person exposures: all 145 occur on the 223 ambiguous queries, and in 94 of those the correct fact ranks first while a same-name competitor fills a later position. Identity therefore depends on the composition of the full block, not only the top-ranked fact.

Restricting every ranker to the same active store isolates temporal validity from ranking. All validity-aware arms have zero observed stale exposure, whereas validity-agnostic arms with participant information expose stale values on 78\% to 80\% of queries. Dense retrieval recalls the current value on almost every query but also retrieves obsolete versions because they remain semantically close to the request. Once the active store and participant information are shared, BM25 + validity + $P_t$ reaches 0.873 clean retrieval against 0.877 for PFM. The paired difference is 0.005 with a 90\% interval from $-0.008$ to 0.019, which lies inside the prespecified $\pm0.02$ equivalence margin ($p_{\mathrm{TOST}}=0.038$); the permutation test alone gives $p=0.55$. Removing participant information lowers the same BM25 arm to 0.774, and the ablation in Table~\ref{tab:ablation} attributes the resulting gap to participant signals rather than to the entity field, graph, usage term, recency, or term cap. Dense and hybrid retrieval with validity remain below the lexical arms because of wrong-person exposure, not current-value recall. Their query encoding also rarely clears a 10~ms serving deadline, with MiniLM taking about 18~ms at the median (Figure~\ref{fig:deadline}).

\subsection{Sensitivity to Key Assignment}
\label{sec:keynoise}

Proposition~\ref{prop:one_active} assumes that revisions of a slot share a key. We violate that premise directly after extraction. With missed merges at rate $p$, stale exposure rises from 0.100 at $p{=}0.05$ to 0.370 at $p{=}0.2$ and 0.592 at $p{=}0.4$, approaching the keyless store's 0.703. False merges produce a different failure: they remove the current value from the active view. At $q{=}0.05$, 6.1\% of queries have lost their current value; at $q{=}0.2$, 22.5\% have, and clean retrieval falls to 0.553 (Table~\ref{tab:keynoise}). Key granularity and arrival order create the same asymmetry. A key that is too coarse merges a person's standup with their sync and reduces current-value recall to 0.500, while delivering 30\% of chains out of statement order lowers clean retrieval by 0.130; ordering those updates by stated time restores the in-order result (Appendix~\ref{app:slices}).

\paragraph{Keys from a model.}\label{sec:extractor} To separate key assignment from the benchmark templates, qwen2:7b rewrites every chain message in three benchmark instances while preserving names and values; 1{,}012 of 1{,}268 rewrites pass the acceptance check and the remainder retain the template wording. Every key assigner sees the same rewritten corpus. We compare the rule extractor of Section~\ref{sec:pfm}, four model-based key assigners ranging from 3B to 14B, a Mem0-shaped model decision over the five most similar active facts~\citep{chhikara2025mem0}, and a keyless store.

Key recall alone does not predict prompt quality. Rewriting lowers the rule extractor's key recall to 0.807 and leaves clean retrieval at 0.781. Each model-based assigner has higher key recall, from 0.980 to 0.996, yet clean retrieval falls between 0.560 and 0.644; every difference from the rule extractor excludes zero under the chain-cluster bootstrap (Table~\ref{tab:extractor}). The best model by clean retrieval, qwen2.5:7b, reaches key precision of 0.990 but remains 0.136 below the rule extractor in clean retrieval (CI 0.086 to 0.186). Model size does not order the results: qwen2.5:14b reaches 0.593 clean retrieval against 0.644 for the 7B model from the same generation.

The error taxonomy explains the reversal. The rule extractor misses 245 keys but never assigns the wrong slot. The four models miss only 5 to 25 keys but misfile 13 to 93 messages, always to the wrong slot of the correct entity. Three models most often confuse room with meeting, reflecting overlapping slot vocabulary in the benchmark (Table~\ref{tab:key_errors}). A missed key leaves both values active and therefore visible in the prompt; a wrong-slot key can close an unrelated current fact and delete the answer silently. In this setting, the precision of slot assignment matters more than revision recall. The public-corpus study tests whether a model can preserve the same key vocabulary when the slot inventory is not supplied.

\subsection{Identity Resolution and Abstention}
\label{sec:idresults}

\begin{table}[t]
\centering
\small
\setlength{\tabcolsep}{4pt}
\begin{tabular}{@{}lrrrrrr@{}}
\toprule
 & \multicolumn{3}{c}{Main benchmark} & Third party & \multicolumn{2}{c}{Hard slice} \\
\cmidrule(lr){2-4}\cmidrule(lr){5-5}\cmidrule(l){6-7}
Serving rule & Wrong-person $\downarrow$ & Clean $\uparrow$ & Abstained $\downarrow$ & Recall $\uparrow$ & Clean $\uparrow$ & Abstained $\uparrow$ \\
\midrule
none & 0.121 & 0.877 & 0.000 & 0.985 & 0.000 & 0.000 \\
relative cutoff @0.7 & 0.008 & 0.983 & 0.000 & 0.975 & 0.908 & 0.000 \\
participant filter & 0.000 & 0.998 & 0.000 & 0.020 & 1.000 & 0.000 \\
identity posterior & 0.032 & 0.967 & 0.000 & \cellcolor{cPos!60}0.985 & \cellcolor{cPos!60}1.000 & 0.000 \\
\quad + slot gate & 0.032 & 0.967 & 0.002 & 0.985 & 1.000 & \cellcolor{cPos!60}1.000 \\
margin gate & 0.097 & 0.743 & 0.158 & 0.815 & 0.000 & 0.000 \\
\bottomrule
\end{tabular}
\caption{Identity resolution and abstention across three query sets: the main benchmark (1{,}200 queries), queries about a third party rather than the interlocutor (200), and a hard same-name slice (120 answerable and 60 unanswerable requests). Wrong-person and clean follow Table~\ref{tab:temporal}. Recall indicates that the current value is present; abstained indicates an empty block. Stale exposure is zero in every row. Shading marks the focal comparisons.}
\label{tab:identity}
\end{table}

A hard participant filter removes all wrong-person exposure on the main benchmark but recalls only 2.0\% of third-party facts. The posterior in Eq.~(\ref{eq:identity}) reduces wrong-person exposure from 0.121 to 0.032 and raises clean retrieval to 0.967 (paired difference 0.089, CI 0.071 to 0.110) while preserving current-value recall at 0.998 overall and 0.985 on third-party facts. On the hard slice, where same-named people have facts with the same slot, noun, template, and arrival day, both the filter and posterior remove wrong-person exposure; a relative score cutoff leaves 0.092. The graph term matters here even though it does not improve ranking in Table~\ref{tab:ablation}: it allows an entity discussed with the current interlocutor to remain eligible even when that entity is not the interlocutor. The hard filter is preferable when only facts about the interlocutor are needed; the posterior preserves the broader third-party use case.

Abstention depends more strongly on schema. On 60 unanswerable requests, every ungated ranker, validity filter, and identity rule returns another person's value of the requested type. The slot gate abstains on all 60 and on only 0.002 of answerable main-benchmark requests, but it relies on the benchmark's six-slot inventory. The schema-free margin gate abstains on none of the 60 unanswerable requests and on 15.8\% of answerable benchmark requests, lowering clean retrieval by 0.134. Query-time model baselines do not remove this dependence (Appendix~\ref{app:systems}). The temporal reranker reduces stale exposure but also discards many current values, the abstention judge catches only two thirds of unanswerable requests, and the model-based entity resolver reaches 0.931 clean retrieval against 0.969 for the posterior while requiring a request-time model call on ambiguous mentions.

\subsection{Evaluation on Public Corpora}
\label{sec:real}

\begin{table}[t]
\centering
\small
\setlength{\tabcolsep}{4pt}
\begin{tabular}{@{}lrrrrrr@{}}
\toprule
 & Both & Merged & Merge & Cross- & & \\
Key assignment & keyed $\uparrow$ & $\mid$ both $\uparrow$ & recall $\uparrow$ & question $\downarrow$ & Stale $\downarrow$ & Clean $\uparrow$ \\
\midrule
stateless, qwen2:7b & 0.438 & 0.000 & 0.000 & 0.042 & 0.688 & 0.219 \\
\quad + store in the prompt & 0.500 & 0.125 & \cellcolor{cPos!60}0.062 & 0.294 & 0.594 & 0.250 \\
\quad + model-judged linking & 0.438 & 0.071 & 0.031 & 0.143 & 0.719 & 0.188 \\
\quad + Jaccard linking, $\theta{=}0.6$ & 0.438 & 0.071 & 0.031 & 0.087 & 0.719 & 0.188 \\
\quad + Jaccard linking, $\theta{=}0.4$ & 0.438 & 0.143 & \cellcolor{cPos!60}0.062 & 0.140 & 0.688 & 0.219 \\
\midrule
stateless, qwen2.5:7b & 0.219 & 0.143 & 0.031 & 0.012 & 0.688 & 0.188 \\
\quad + model-judged linking & 0.219 & 0.286 & \cellcolor{cPos!60}0.062 & 0.029 & 0.719 & 0.156 \\
\midrule
no update resolution & -- & -- & -- & -- & 0.719 & 0.156 \\
BM25 & -- & -- & -- & -- & 0.812 & 0.156 \\
BM25 + recency & -- & -- & -- & -- & 0.688 & 0.250 \\
\bottomrule
\end{tabular}
\caption{Open-domain key assignment on 32 LongMemEval knowledge-update questions over 2{,}369 turns. Both keyed is the share of questions whose two answer-bearing turns both receive keys; merged $\mid$ both is the conditional share receiving the same key; merge recall is their product; cross-question is the share of keys reused across different questions. Stale and clean follow Table~\ref{tab:temporal}; their intervals are about $\pm0.16$ and do not separate the arms. Shading marks the largest observed merge-recall values.}
\label{tab:lme}
\end{table}

The controlled benchmark supplies a fixed slot vocabulary and consistent entity names. LongMemEval removes that structure. Of its 78 knowledge-update questions, 32 satisfy the requirement that both earlier and later answer spans are verbatim and distinct, allowing prompt-level stale exposure to be labeled exactly.

Stateless key assignment almost never reuses a key across the two turns that define a knowledge update. For qwen2:7b, 1{,}198 turns receive keys but those keys contain 1{,}025 distinct normalized forms, and none of the 32 questions assigns the same key to both answer-bearing turns. Exposing existing store keys to the assigner increases merge recall only to 0.062 and raises cross-question merging from 0.042 to 0.294. Post-extraction model linking and Jaccard linking produce the same tradeoff: conditional merging improves modestly, but false merging also rises. Across all variants, merge recall is at most 0.062, while stale exposure remains between 0.594 and 0.719. BM25 with recency, which requires no keys, reaches 0.250 clean retrieval against 0.156 with no update resolution. On this subset of LongMemEval, the limiting factor is therefore consistent open-domain key discovery rather than the supersession rule itself.

LoCoMo exposes a different limitation. In a single store containing all ten conversations, 0.543 of prompts for questions about a speaker whose first name is shared contain a fact from another such speaker's conversation. A participant filter lowers this exposure to 0.301 without recall loss on this corpus because LoCoMo questions concern the speakers themselves. The posterior in Eq.~(\ref{eq:identity}) cannot apply when two stored people have the identical normalized string \texttt{john}; there are no distinct candidate entities for it to choose between. A graph-affinity substitute reduces the share of off-conversation lines but leaves name-sharing exposure unchanged. The schema-free margin gate likewise trades abstention for recall: at $\tau=1.5$ it catches 95.0\% of 1{,}800 unanswerable requests but abstains on 73.9\% of answerable ones, while the slot gate cannot run without a slot inventory (Appendix~\ref{app:locomo}).

\subsection{Response-Level Effects}
\label{sec:llm}

Prompt exposure is an intermediate outcome, so we also pass all 240 seed-7 queries through two frozen models, qwen2:7b and llama3.2:3b (4-bit, greedy), once per memory arm (Tables~\ref{tab:llm},~\ref{tab:llm_llama}). In the keyless store, 154 prompts contain a superseded value; the 7B model repeats one in 45 completions and the 3B model in 67. With keyed memory, 92.9\% of 7B completions contain the current value, compared with 70.4\% under the keyless store (paired difference 0.225, CI 0.161 to 0.296). For the 3B model, the corresponding rates are 0.912 and 0.592 (difference 0.321, CI 0.246 to 0.402). Telling the model explicitly to use the most recent dated memory line raises BM25 correctness from 0.629 to 0.654, still below 0.842 when BM25 retrieves only active facts. Appendix~\ref{app:llm} reproduces the qualitative gap under key noise, out-of-order arrival, hard identity collisions, and sampled decoding. The memory block also dominates end-to-end latency on this hardware: median time to first token rises from 325~ms without memory to 1.22~s with the block, roughly three orders of magnitude more than the lexical retrieval that produced it.

\section{Limitations}

\paragraph{Data.} The controlled benchmark uses a small template family, globally unique value tokens, correct participant metadata, and a known slot inventory. These choices make prompt exposure identifiable but simplify the language and structure faced by the memory system. The LongMemEval study is also selective: key assignment requires a model call per ingested turn, so we use 2{,}587 turns rather than the full corpus, and the verbatim-span requirement leaves 32 of 78 knowledge-update questions. Questions answered only by paraphrase are therefore excluded, which favors string-based labeling. On LoCoMo, lexical sentence retrieval rarely places the gold answer in the block, so we use that corpus to study exposure by provenance rather than to estimate answer recall. Neither public corpus represents a real user's personal store, and none of the experiments measures user benefit.

\paragraph{Scope of the claims.} Proposition~\ref{prop:one_active} guarantees an update invariant, not semantic correctness. The stale-exposure results are conditional on assigning revisions to the correct slot and entity. Our open-domain experiments test several plausible key-assignment and linking procedures, including exposing existing store keys to the assigner, but they do not rule out a system trained specifically for consistent key discovery or a resolver operating over a richer schema. The identity posterior applies only when competing people are represented as distinct stored entities, and the abstention rule that succeeds on the controlled benchmark assumes a known slot inventory. The serving bound in Proposition~\ref{prop:waiting} is conditional on runtime handoff and scheduling delays; all latency measurements come from one machine and one CPython implementation. Finally, the response-level study uses two small quantized models and a containment-based judge.

\section{Conclusion}

Prompt-level evaluation separates failures that final-answer accuracy can conflate. In our controlled experiments, temporal validity is determined primarily by whether the write path maintains the correct active state: once all rankers receive the same active store and participant information, simple lexical retrieval is equivalent to the reference ranker within the prespecified margin. The guarantee is only as useful as the key assignment that activates it. Missed merges leave obsolete values visible, while false merges silently remove current values; model-based key assigners can achieve high revision recall and still produce worse prompts when slot precision is imperfect. On LongMemEval, consistent open-domain key reuse remains rare even when existing keys are made available to the assigner.

Misattribution and abstention remain separate serving problems. Participant filtering can remove identity errors when the desired fact concerns the interlocutor, but it discards third-party facts; the entity posterior preserves those facts on controlled data and fails when distinct people share the same stored name. A schema-aware gate can abstain reliably, whereas a schema-free margin rule does not transfer to LoCoMo. A useful evaluation therefore begins at the point where memory enters the model: whether the prompt contains a current fact for the intended entity, whether the system declines when no answer is stored, and whether the block arrives in time. Final-answer accuracy remains necessary, but by itself it does not identify which part of the memory system failed.

\section*{Ethics Statement}
\label{sec:ethics}

Personal memory can contain messages, commitments, and information about people other than the user. Supersession retains outdated values as history for provenance, but a deployed system must treat deletion separately: retiring a fact from retrieval is not equivalent to erasing the stored fact or the graph associations derived from it. The reference implementation does not implement complete deletion, so it should not be interpreted as satisfying a deletion request. Wrong-person exposure also creates a cross-context disclosure risk because a fact from one relationship can enter a prompt addressed to another person with the same name. The identity rules studied here are ranking rules, not access-control mechanisms. A deployment that prohibits cross-person disclosure requires an explicit policy enforced when the memory block is assembled. All controlled experiments use synthetic text and invented names. The LongMemEval and LoCoMo analyses use public benchmark corpora; no private user data or human-subject data are collected in this study.

\section*{Reproducibility Statement}

The anonymized repository contains the reference implementation, benchmark generator, public-corpus loaders, fixed configuration, unit tests for the update invariants, cached model calls, per-query logs, and the materials needed to reproduce the reported tables and figures. Model-assisted labels are cached by prompt so the public-corpus analyses can be reproduced without rerunning model inference. Appendix~\ref{app:config} specifies all parameters and seeds; Appendix~\ref{app:bench} documents benchmark construction and counts; Appendix~\ref{app:real} describes public-corpus construction and label auditing; Appendix~\ref{app:prompt} gives the memory-block and response prompts; and Appendix~\ref{app:timing} reports the timing protocol, hardware, and software versions. LongMemEval and LoCoMo are obtained from their original public sources and are not redistributed.

\section*{Use of Large Language Models}

LLM-based assistants were used for code implementation and language editing. The authors designed the experiments, reviewed the resulting code, verified every reported result, and take responsibility for the manuscript.

Open-weight language models also serve as objects of study and experimental components. They generate completions in Section~\ref{sec:llm}, assign keys in Sections~\ref{sec:extractor} and~\ref{sec:real}, and extract verbatim answer spans for the LongMemEval labels. The latter use is restricted to spans copied from annotated turns; spans that fail the verbatim check are discarded, and Appendix~\ref{app:real} documents the filtering and label audit. This restriction prevents model-generated content from being treated as ground truth.

\bibliography{iclr2027_conference}

@inproceedings{park2023generative,
  title={Generative agents: Interactive simulacra of human behavior},
  author={Park, Joon Sung and O'Brien, Joseph and Cai, Carrie Jun and Morris, Meredith Ringel and Liang, Percy and Bernstein, Michael S},
  booktitle={Proceedings of the 36th annual acm symposium on user interface software and technology},
  pages={1--22},
  year={2023}
}

@article{chhikara2025mem0,
  title={Mem0: Building production-ready ai agents with scalable long-term memory},
  author={Chhikara, Prateek and Khant, Dev and Aryan, Saket and Singh, Taranjeet and Yadav, Deshraj},
  journal={arXiv preprint arXiv:2504.19413},
  year={2025}
}

@article{rasmussen2025zep,
  title={Zep: a temporal knowledge graph architecture for agent memory},
  author={Rasmussen, Preston and Paliychuk, Pavlo and Beauvais, Travis and Ryan, Jack and Chalef, Daniel},
  journal={arXiv preprint arXiv:2501.13956},
  year={2025}
}

@article{gutierrez2024hipporag,
  title={Hipporag: Neurobiologically inspired long-term memory for large language models},
  author={Guti{\'e}rrez, Bernal J and Shu, Yiheng and Gu, Yu and Yasunaga, Michihiro and Su, Yu},
  journal={Advances in neural information processing systems},
  volume={37},
  pages={59532--59569},
  year={2024}
}

@article{yadav2026temporal,
  title={Temporal Validity in Retrieval Memory: Eliminating Stale-Fact Errors for AI Agents over Evolving Knowledge},
  author={Yadav, Neeraj},
  journal={arXiv preprint arXiv:2606.26511},
  year={2026}
}

@article{wang2026toki,
  title={TOKI: A Bitemporal Operator Algebra for Contradiction Resolution in LLM-Agent Persistent Memory},
  author={Wang, Ziming},
  journal={arXiv preprint arXiv:2606.06240},
  year={2026}
}

@article{xu2026mem,
  title={A-mem: Agentic memory for llm agents},
  author={Xu, Wujiang and Liang, Zujie and Mei, Kai and Gao, Hang and Tan, Juntao and Zhang, Yongfeng},
  journal={Advances in Neural Information Processing Systems},
  volume={38},
  pages={17577--17604},
  year={2026}
}

@article{packer2023memgpt,
  title={MemGPT: towards LLMs as operating systems.},
  author={Packer, Charles and Fang, Vivian and Patil, Shishir\_G and Lin, Kevin and Wooders, Sarah and Gonzalez, Joseph\_E},
  year={2023},
  publisher={ArXiv}
}

@inproceedings{zhong2024memorybank,
  title={Memorybank: Enhancing large language models with long-term memory},
  author={Zhong, Wanjun and Guo, Lianghong and Gao, Qiqi and Ye, He and Wang, Yanlin},
  booktitle={Proceedings of the AAAI conference on artificial intelligence},
  volume={38},
  number={17},
  pages={19724--19731},
  year={2024}
}

@inproceedings{xia2026memora,
  title={Memora: A harmonic memory representation balancing abstraction and specificity},
  author={Xia, Menglin and Zhang, Xuchao and Dixit, Shantanu and Harimurugan, Paramaguru and Wang, Rujia and R{\"u}hle, Victor and Sim, Robert and Bansal, Chetan and Rajmohan, Saravan},
  booktitle={Forty-third International Conference on Machine Learning},
  year={2026}
}

@article{zhou2026we,
  title={Are We Ready For An Agent-Native Memory System?},
  author={Zhou, Wei and Zhou, Xuanhe and Han, Shaokun and Xu, Hongming and Li, Guoliang and Li, Zhiyu and Xiong, Feiyu and Wu, Fan},
  journal={arXiv preprint arXiv:2606.24775},
  year={2026}
}

@inproceedings{banerjee2026apex,
  title={APEX-MEM: Agentic Semi-Structured Memory with Temporal Reasoning for Long-Term Conversational AI},
  author={Banerjee, Pratyay and Moshtaghi, Masud and Subramanian, Shivashankar and Misra, Amita and Chadha, Ankit},
  booktitle={Proceedings of the 64th Annual Meeting of the Association for Computational Linguistics (Volume 1: Long Papers)},
  pages={16470--16489},
  year={2026}
}

@article{su2026beyond,
  title={Beyond Dialogue Time: Temporal Semantic Memory for Personalized LLM Agents},
  author={Su, Miao and Guo, Yucan and Hou, Zhongni and Bai, Long and Li, Zixuan and Zhang, Yufei and Yin, Guojun and Lin, Wei and Jin, Xiaolong and Guo, Jiafeng and others},
  journal={arXiv preprint arXiv:2601.07468},
  year={2026}
}

@article{jensen1999temporal,
  title={Temporal data management},
  author={Jensen, Christian S and Snodgrass, Richard T},
  journal={IEEE Transactions on knowledge and data engineering},
  volume={11},
  number={1},
  pages={36--44},
  year={1999},
  publisher={IEEE}
}

@inproceedings{mckenney1998read,
  title={Read-copy update: Using execution history to solve concurrency problems},
  author={McKenney, Paul E and Slingwine, John D},
  booktitle={Parallel and Distributed Computing and Systems},
  volume={509518},
  pages={509--518},
  year={1998}
}

@book{gospodnetic2010lucene,
  title={Lucene in action},
  author={Gospodnetic, Otis and Hatcher, Erik and McCandless, Michael},
  year={2010},
  publisher={Simon and Schuster}
}

@article{zhao2025ame,
  title={AME: An Efficient Heterogeneous Agentic Memory Engine for Smartphones},
  author={Zhao, Xinkui and Ma, Qingyu and Zhang, Yifan and Lou, Hengxuan and Cheng, Guanjie and Deng, Shuiguang and Yin, Jianwei},
  journal={arXiv preprint arXiv:2511.19192},
  year={2025}
}

@article{yuan2026agentir,
  title={AgentIR: A Workload-Adaptive Cascade Retrieval Substrate for Long-Term Conversational Memory},
  author={Yuan, Aojie and Zhang, Haiyue and Nazarian, Shahin},
  journal={arXiv preprint arXiv:2605.25092},
  year={2026}
}

@inproceedings{maharana2024evaluating,
  title={Evaluating very long-term conversational memory of llm agents},
  author={Maharana, Adyasha and Lee, Dong-Ho and Tulyakov, Sergey and Bansal, Mohit and Barbieri, Francesco and Fang, Yuwei},
  booktitle={Proceedings of the 62nd Annual Meeting of the Association for Computational Linguistics (Volume 1: Long Papers)},
  pages={13851--13870},
  year={2024}
}

@article{wu2024longmemeval,
  title={Longmemeval: Benchmarking chat assistants on long-term interactive memory},
  author={Wu, Di and Wang, Hongwei and Yu, Wenhao and Zhang, Yuwei and Chang, Kai-Wei and Yu, Dong},
  journal={arXiv preprint arXiv:2410.10813},
  year={2024}
}

@inproceedings{hu2026evaluating,
  title={Evaluating memory in llm agents via incremental multi-turn interactions},
  author={Hu, Yuanzhe and Wang, Yu and McAuley, Julian},
  booktitle={International Conference on Learning Representations},
  volume={2026},
  pages={156259--156291},
  year={2026}
}

@inproceedings{uddin2026recall,
  title={From recall to forgetting: Benchmarking long-term memory for personalized agents},
  author={Uddin, Md Nayem and Shubham, Kumar and Blanco, Eduardo and Baral, Chitta and Wang, Gengyu},
  booktitle={Findings of the Association for Computational Linguistics: ACL 2026},
  pages={26814--26841},
  year={2026}
}

@article{chao2026stale,
  title={STALE: Can LLM Agents Know When Their Memories Are No Longer Valid?},
  author={Chao, Hanxiang and Bai, Yihan and Sheng, Rui and Li, Tianle and Sun, Yushi},
  journal={arXiv preprint arXiv:2605.06527},
  year={2026}
}

@article{tao2026memconflict,
  title={MemConflict: Evaluating Long-Term Memory Systems Under Memory Conflicts},
  author={Tao, Zhen and Zhao, Jinxiang and Liu, Peng and Xi, Dinghao and Chen, Yanfang and Xu, Wei and Li, Zhiyu},
  journal={arXiv preprint arXiv:2605.20926},
  year={2026}
}

@article{long2026memtrace,
  title={MemTrace: Probing What Final Accuracy Misses in Long-Term Memory},
  author={Long, Xianxuan and Chen, Zhikai and Zeng, Shenglai and Wang, Shouren and Guo, Kai and Tang, Jiliang},
  journal={arXiv preprint arXiv:2606.17328},
  year={2026}
}

@book{robertson2009probabilistic,
  title={The probabilistic relevance framework: BM25 and beyond},
  author={Robertson, Stephen and Zaragoza, Hugo},
  volume={4},
  year={2009},
  publisher={Now Publishers Inc}
}

@inproceedings{robertson2004simple,
  title={Simple BM25 extension to multiple weighted fields},
  author={Robertson, Stephen and Zaragoza, Hugo and Taylor, Michael},
  booktitle={Proceedings of the thirteenth ACM international conference on Information and knowledge management},
  pages={42--49},
  year={2004}
}

@article{schuirmann1987comparison,
  title={A comparison of the two one-sided tests procedure and the power approach for assessing the equivalence of average bioavailability},
  author={Schuirmann, Donald J},
  journal={Journal of pharmacokinetics and biopharmaceutics},
  volume={15},
  number={6},
  pages={657--680},
  year={1987},
  publisher={Springer}
}
\bibliographystyle{iclr2027_conference}

\appendix

\section{Configuration}
\label{app:config}

\begin{table}[h]
\centering
\small
\begin{tabular}{@{}llr@{}}
\toprule
Component & Parameter & Value \\
\midrule
Selection & facts per request $k$; character budget $B$ (whole block) & 6; 800 \\
 & shortlist before selection & 50 \\
Query & positional decay $\lambda$; participant weight $\beta$ (additive) & 0.95; 1.5 \\
 & term cap $m$ & 12 \\
BM25F & field weights $\alpha$ (text, entities, participants) & 1.0, 2.5, 2.0 \\
 & $k_1$; $b$ (shared by all fields) & 1.2; 0.75 \\
Score & recency floor $\varepsilon$; half-life $T_{\mathrm{rec}}$ & 0.15; 30 days \\
 & usage weight $\eta$; usage half-life & 0.35; 14 days \\
 & graph weight $\gamma$ & 0.8 \\
Graph & edge half-life $T_g$; neighbors stored per node & 30 days; 64 \\
 & damping $\rho$; non-seed nodes activated; participant seeds & 0.5; 12; yes \\
Admission & Jaccard threshold; value-token match; usage inheritance & 0.8; yes; yes \\
Pruning & retention threshold; graph edge threshold & 0.02; 0.05 \\
\bottomrule
\end{tabular}
\caption{Frozen parameters, identical in every experiment. Ablations change one entry at a time.}
\label{tab:config}
\end{table}

The BM25F score of the lexical stage is
\begin{equation}
\mathrm{BM25F}(q_t,f)=\sum_w q_t(w)\,\mathrm{idf}(w)\,\frac{(k_1+1)\,\widetilde{\mathrm{tf}}(w,f)}{\widetilde{\mathrm{tf}}(w,f)+k_1\bigl(1-b+b\,\mathrm{dl}_f/\overline{\mathrm{dl}}\bigr)},
\end{equation}
where $\widetilde{\mathrm{tf}}(w,f)=\sum_F\alpha_F\,\mathrm{tf}_F(w,f)$ merges the text, entity, and participant fields with the weights above, $\mathrm{dl}_f=\sum_w\widetilde{\mathrm{tf}}(w,f)$, and $\mathrm{idf}(w)=\log\bigl(1+(N-n_w+0.5)/(n_w+0.5)\bigr)$ over the $N$ active facts.

All experiments use one fixed configuration. Benchmark seeds are 7, 11, 13, 17, and 19; seed 7 is used for the latency, snapshot-load, response-level, and rewritten-corpus studies. Supersession orders updates by arrival unless stated otherwise; the event-time variant of Section~\ref{sec:slices} instead preserves the value with the latest stated time.
\section{Benchmark}
\label{app:bench}

\paragraph{Chains and messages.} Chains 0 to 11 are assigned to six pairs of people who share a first name, and both members of a pair receive a chain on the same person slot. Later chains draw an unused (entity, slot) pair at random, cycling over slot types, so they can also collide with a first name. A chain with $r$ revisions has $r+1$ messages with ages drawn so that later messages are more recent (between 0.3 and 27 days before the query time). Each message uses one of three templates per slot, for example ``The \{noun\} with \{entity\} is moved to \{value\}.'', ``\{entity\} \{noun\} is rescheduled to \{value\}.'', and ``New time for the \{noun\} with \{entity\}: moved to \{value\}.'' for meetings. When chain $i$ satisfies $i \bmod 10=7$ and has a revision, all its messages use the first template (verbatim revisions). A cancelled chain ends with ``The \{noun\} with \{entity\} is cancelled.'' and its current value is the cancellation. Each message's participant is the person for person slots and a fixed randomly chosen person for project slots. Filler messages come from four templates (ticket assignments, deployments, hardware maintenance, license reservations) with bot senders as participants and eight names that do not occur in chains, and twelve short chit-chat messages to random people are added.

\paragraph{Queries and labels.} Query templates end immediately before the value, for example with the phrase ``the standup with Alex is at'' or ``Alex standup is now at''. The labels of a query are the chain's current value, its superseded values (all values for a cancelled chain), and the current values of chains on the same slot whose person has the same first name. Values are drawn from type-specific generators and rejected if their normalized form was used before, so each value belongs to one chain. Normalization lowercases, abbreviates month names, removes digit-grouping commas, and removes the space before am/pm. Containment requires token boundaries, so ``Dec 2'' does not match ``Dec 24'' and prefix matches do not create false labels.

\paragraph{Slices.} The slices of Section~\ref{sec:slices} reuse the same generator. \emph{Multi-valued}: 24 people, each with a standup and a sync chain of two revisions, 240 queries over five seeds, plus 200 filler messages. \emph{Out of order}: the main corpus with the messages of 30\% of the chains permuted in arrival order, their stated times unchanged (37 to 43 chains per seed). \emph{Participants}: the main corpus with each chain assigned in turn to one of four cases, the participant being the subject, another person, a pair including the subject, or absent; queries keep the subject as $P_t$. \emph{Hard identity}: 12 pairs of people who share a first name, each pair holding a standup chain with the same noun, the same template, and start times within a day, plus one query per pair about a room that nobody mentioned, whose labels are all meeting values in the corpus.

\paragraph{Rewritten corpus.} Section~\ref{sec:extractor} rewrites the seed 7, 11 and 13 corpora without fillers (1{,}316 messages, 1{,}268 of them in chains; 349 of 429, 333 of 425 and 330 of 414 rewrites accepted). One model (qwen2:7b) does the rewriting for every arm, so a key assigner is never compared against text of its own making, and the four assigners of Table~\ref{tab:extractor} read identical corpora. A rewrite is accepted when every name survives and every value survives, after restoring respellings with the same rules the judge uses; otherwise the template sentence stays. Key assignment is scored against the generator's (entity, slot) keys and pooled over seeds, with the error taxonomy of Table~\ref{tab:key_errors} counted per message; the model-resolved path shows the model the five highest-scoring active facts by PFM's own ranking.

\begin{table}[h]
\centering
\scriptsize
\setlength{\tabcolsep}{3pt}
\begin{tabular}{@{}lrrrrrrrrrrrrr@{}}
\toprule
Seed & Msgs & Queries & Amb. chains & Amb. & Rev. & Canc. & Verb. & Para. & $d{=}0$ & $d{=}1$ & $d{=}2$ & $d{=}3$ & $d{=}4$ \\
\midrule
s7 & 1441 & 240 & 33 & 43 & 188 & 14 & 14 & 80 & 52 & 81 & 48 & 34 & 25 \\
s11 & 1437 & 240 & 32 & 40 & 196 & 16 & 13 & 80 & 44 & 81 & 54 & 30 & 31 \\
s13 & 1426 & 240 & 34 & 48 & 189 & 6 & 14 & 80 & 51 & 70 & 73 & 23 & 23 \\
s17 & 1465 & 240 & 31 & 42 & 196 & 16 & 13 & 80 & 44 & 59 & 66 & 39 & 32 \\
s19 & 1428 & 240 & 38 & 50 & 193 & 21 & 12 & 80 & 47 & 69 & 77 & 25 & 22 \\
\bottomrule
\end{tabular}
\caption{Benchmark accounting per seed (filler corpus). Queries per chain: 80 chains have two queries and 80 have one. Columns after ``Amb. chains'' count queries: ambiguous, revised, cancelled, verbatim-revision, paraphrased query template, and revision depth $d$.}
\label{tab:counts}
\end{table}
\section{Memory Block and Prompt}
\label{app:prompt}

A block starts with the header \texttt{[Relevant memory]} and has one line per fact, \texttt{- (2025-06-11, to Alex Rivera) The standup with Alex Rivera is moved to room 16C.}, with the date of the source observation and its participants. All arms use this format and the greedy selection of Section~\ref{sec:pfm}. The response-level prompt is:

\begin{quote}\small\ttfamily
[request 1804289383]\\
You are helping the user write a message to \{participant\}. Continue the unfinished message with only the words that come next.\\[2pt]
\{latest-value instruction, one arm only\}\{memory block\}\\[2pt]
Unfinished message to \{participant\}:\\
\{query prefix\}
\end{quote}

The first line is a random request id that prevents the server from reusing a cached prompt prefix, so every TTFT includes a full prefill. The latest-value instruction reads: ``Memory lines are dated. If two lines give different values for the same thing, use the most recent one.'' The judge normalizes the completion as above and assigns one of five outcomes: \emph{correct} (current value only), \emph{mixed} (current value together with a superseded or other-person value), \emph{stale}, \emph{wrong} (other-person value only), and \emph{other}. For a cancelled chain any form of ``cancel'' counts as the current value.
\section{Identity Resolution and Abstention Details}
\label{app:identity}

\paragraph{Candidates and aliases.} Mentions are the alphabetic, non-stopword tokens of the request. A mention is ambiguous when two or more stored entity strings contain it, and the candidate set is the union of those strings. Two candidates are \emph{aliases} rather than rivals when one's token set contains the other's, which is how the extractor's ``halcyon'' and ``the halcyon'' come to name one project; aliases are pooled into a single outcome of Eq.~(\ref{eq:identity}) and share its probability, so only genuinely competing names are resolved. A request whose mentions all resolve to one alias group is left untouched.

\paragraph{Slot detector.} The gate reads the requested slot from its noun when the noun is unambiguous (deadline, due date, launch, cutover, day rate, budget cap) and from its trailing cue otherwise, because the benchmark's room and meeting slots share nouns: ``the standup with Alex is in'' asks for a room, whereas ``\ldots is at'' asks for a time. The rule agrees with the generator's slot on all 1{,}200 main-benchmark queries and all 180 hard-slice queries. Because this accuracy depends on the benchmark's fixed slot inventory, we also evaluate the schema-free margin gate.

\paragraph{Thresholds.} The posterior threshold $\tau$ of the slot gate makes no difference between 0.3 and 0.5 (clean retrieval 0.967, abstention 0.002 on the main benchmark); at 0.7 it starts rejecting correct facts whose candidate is not clearly ahead, and clean retrieval falls to 0.712 with 25.7\% abstention. The margin gate was swept at $\tau\in\{1.5,2,3,5\}$ on both kinds of corpus; on the controlled benchmark every setting abstains on 15.8 to 17.8\% of answerable requests and on none of the unanswerable ones, and Section~\ref{sec:real} reports the sweep on LoCoMo.
\section{Public Corpora and Their Labels}
\label{app:real}

\paragraph{LongMemEval.} We use the oracle release, which associates each question with the sessions containing its answer, and take its 78 knowledge-update questions. Each question annotates two user turns as answer-bearing: the later turn states the value in force and the earlier turn the value it replaced. To obtain span-level labels for exposure, qwen2:7b is asked once per annotated turn to copy, verbatim and in at most eight words, the span that answered the question at that time. A question is retained only when both spans occur verbatim in their source turns after the normalization of Appendix~\ref{app:bench}, contain at most ten tokens, and differ from one another. Excluded questions and their reasons are recorded in the released label audit. We do not use the dataset's answer string as a containment label because it is often a paraphrase rather than a source span.

The corpus the store ingests is every user turn of those 78 questions plus every user turn of 160 further questions drawn at random from the other five question types, which serve as distractors. We cap the distractors because assigning an open-domain key costs one model call per turn. Assistant turns are not ingested: the memory is of what the user said.

\paragraph{Open-domain keys.} The rule extractor of Section~\ref{sec:pfm} does not transfer: it requires a signal verb and an anchor and keeps 0.29 sentences per turn on this corpus, compared with 2.75 for a rule that keeps every sentence. The public-corpus stores therefore use the latter, with keys assigned by a model. For each turn the model returns an entity, attribute, and value and is instructed to reuse the same wording when a property recurs; the normalized (entity, attribute) pair becomes the key for the sentence carrying that value. We evaluate three assignment procedures. \emph{Stateless} applies this prompt once per turn with qwen2:7b or qwen2.5:7b. \emph{Store in the extraction prompt} adds the ten existing keys with the largest token overlap and asks the model to copy an existing pair when appropriate. \emph{Model-judged linking} leaves extraction unchanged and adds a resolution step in which a new key that shares a token with existing keys is compared with up to ten candidates. This procedure keys exactly the same turns as its stateless counterpart, so the comparison isolates linking rather than coverage. We also compare a Jaccard linker that merges keys by token overlap without a model call.

Three diagnostics then measure key quality without a key annotation. \emph{Both turns keyed} is the share of knowledge-update questions whose two annotated turns each received a key at all, which bounds what any resolver can merge. \emph{Merged given both} is the share of those that received the same key, which is the merge supersession needs. The \emph{cross-question rate} is the share of keys shared by turns belonging to different questions, which is the false merge that deletes a value. Merge recall, the product of the first two, is reported as well because it is the quantity a reader expects, but it cannot be compared across arms that key different shares of the corpus, and these do: 1{,}198, 1{,}657 and 438 of 2{,}369 turns.

\paragraph{LoCoMo.} We ingest every dialogue turn of all ten conversations into one store, prefixing each turn with its speaker so the speaker is part of the record and an entity of the fact. Questions are those of the single-hop, multi-hop and temporal categories whose gold answer occurs verbatim in their own conversation; the rest are skipped and counted. Misattribution needs no value label here: each fact carries the conversation it came from, and a block that answers a question from one conversation with a fact from another whose speakers share a first name is a misattribution by provenance. Three of the ten conversations have a speaker named John. For the abstention measurement we rebuild one store per conversation and ask it both its own questions and twenty questions from each other conversation, which have no answer in it.
\section{Timing Protocol and Environment}
\label{app:timing}

Hardware: Apple M2 (4 performance and 4 efficiency cores), 16~GB RAM, macOS 26.5. Software: CPython 3.11.1, PyTorch 2.10.0, sentence-transformers 5.3.0, NumPy 2.4.3. Retrieval latency is measured with \texttt{time.perf\_counter} around ranking, selection, and block formatting; the store is built before timing and no construction runs concurrently. Dense retrieval uses all-MiniLM-L6-v2 on the CPU with 4 PyTorch threads, batch size 1, and fact embeddings computed at build time (1.2~s for about 1{,}400 facts); query embeddings are not cached. For each of the 240 seed-7 queries and each of 5 repeats, all 15 arms run in a freshly shuffled order after a warm-up pass, and percentiles use all 1{,}200 runs per arm. The snapshot read is timed as a reference load of a published result. MiniLM query encoding accounts for nearly all dense latency in this protocol (median 17.6~ms of 17.9~ms). The response-level study uses ollama 0.34.0 serving qwen2:7b and llama3.2:3b (both Q4\_0, fully on the GPU, context 4{,}096), temperature 0, seed 0, and at most 24 generated tokens; the sampling check of Section~\ref{sec:llm} repeats the comparison at temperature 0.7 with seeds 1, 2, and 3. The rewriting and key-assignment study calls qwen2:7b with at most 80 tokens for a rewrite, 48 for a key, and 8 for an update decision. The public-corpus study of Section~\ref{sec:real} uses at most 32 tokens for an open-domain key and 32 for a label span, while the query-time baselines of Section~\ref{sec:systems} use 24 for reranking and 8 for resolution or abstention. Model outputs are cached by model, length, and prompt and included with the anonymous materials, so reproducing the reported analyses does not require rerunning inference. Request-path model latency is timed around the model call with the model already resident. The machine holds one evaluated model at a time; loading qwen2.5:14b takes about 95~s and is excluded from all reported request-path latency. The model ladder in Section~\ref{sec:extractor} is therefore evaluated one model at a time.
\section{Serving Under Snapshot Lag and Load}
\label{app:serving}

\paragraph{Protocol and bound.}

In snapshot mode, a background worker periodically runs retrieval and publishes an immutable result by assigning one reference. A request reads that reference. Each snapshot records the store version, the request context, and when it was produced. In fresh mode, the request hands its context to a retrieval worker that is already running and waits at most $\omega_t<D_t$; on timeout it sets a cancellation flag, which the worker checks between query terms and before scoring, and returns the latest snapshot. Workers acquire the store lock; the request thread never does.

\begin{proposition}[Bound on memory-induced waiting]
\label{prop:waiting}
Assume (i) snapshots are immutable and published by a single reference assignment, (ii) the request thread acquires no lock that construction or a retrieval worker holds, (iii) reading the reference takes at most $r$, (iv) handing a request to an already-running worker takes at most $h$, a free worker being available, and (v) the request thread runs within $\sigma$ of becoming runnable, at every point on this path. Then snapshot mode adds at most $\sigma+r$ to the response path, and fresh mode adds at most $h+\omega_t+\sigma+r$, independently of how long extraction, index maintenance, graph updates, or pruning take.
\end{proposition}

Appendix~\ref{app:proofs} gives the proof. Assumptions (iv) and (v) concern the runtime rather than the publication protocol. Starting a thread for each request is unsuitable here because \texttt{Thread.start} blocks until the new thread runs; under interpreter-lock contention this cost has a median of $3.0$~ms and a maximum of $12.3$~ms, already exceeding the intended waiting budget. We therefore use prestarted workers and a queue, for which the median handoff cost is $7~\mu$s. More than one worker is also required because a single worker can create head-of-line blocking while a cancelled retrieval drains. Assumption (v) requires the scheduling delay on the serving path itself, not a proxy measured elsewhere, so Appendix~\ref{app:serving} reports that quantity directly. The proposition bounds waiting, not freshness or model latency: snapshot age is unbounded unless publication runs on a fixed period, and a snapshot built before the latest revision exposes the old value.

\paragraph{Lag.} For each query we build the snapshot on the store as it was just before the last one or two messages of the query's own chain arrived (for a one-message chain, before that message), and judge it against the current labels. Table~\ref{tab:lag} shows that a snapshot one relevant update old exposes the superseded value on every revised query and lacks the fact entirely on unrevised ones. Snapshot mode is always on time but is only as fresh as its last publication.

\paragraph{Load.} A store of 20{,}000 synthetic facts (Appendix~\ref{app:scaling}) serves 2{,}000 requests while a construction thread ingests facts at up to 400 per second and prunes once per second, and a snapshot worker publishes every 50~ms. Each request reads the snapshot and then issues a fresh request with a 5~ms waiting budget, both timed from the request thread, with the garbage collector enabled and disabled. Each run ingested about 2{,}500 facts and completed 22 to 23 pruning passes.

Snapshot mode behaves as the bound predicts: the reference read $r$ is 2~$\mu$s at the median and 30~$\mu$s at its maximum. Fresh mode requires the scheduling delay on the serving path itself. A 2~ms sleep measured elsewhere in the request loop gives a much smaller quantity and does not bound the observed fresh waiting. Two effects explain the difference. Starting a retrieval worker blocks the caller until the new thread runs, costing 3.0~ms at the median and 12.3~ms at the maximum under this load; additional delay arises when the request thread is descheduled while another worker holds the interpreter lock.

Workers are therefore started once and fed through a queue, which brings the hand-off term $h$ to 17~$\mu$s at the median and 58~$\mu$s at its worst. There are four of them rather than one, because with a single worker a request queues behind whatever cancelled retrieval is still draining: that alone took the fresh-mode fallback rate from 12\% to 49\%, which is head-of-line blocking the bound does not name. With the pool, the worker finishes inside the 5~ms budget on 81\% of requests with the collector enabled and 88\% with it disabled.

The scheduling delay required by the bound is $\sigma=2.1$~ms at the median, 28.5~ms at the 99th percentile, and 34.8~ms at the maximum with the collector enabled; without the collector, the corresponding values are 1.8, 10.6, and 16.0~ms. A sleep-based proxy gives 1.0, 8.5, and 32.2~ms and understates the required 99th-percentile delay by about a factor of three. Proposition~\ref{prop:waiting} is conditional on this runtime quantity, which is about 35~ms at the observed maximum. The interpreter lock and garbage collector determine the tail in this implementation; a runtime with different scheduling behavior would require its own measurement of $\sigma$.

\begin{table}[h]
\centering
\small
\begin{tabular}{@{}llrrr@{}}
\toprule
Lag & Slice & Current & Stale & Clean \\
\midrule
0 & revised & 1.000 & 0.000 & 0.877 \\
0 & unrevised & 0.992 & 0.000 & 0.878 \\
1 & revised & 0.021 & 1.000 & 0.000 \\
1 & unrevised & 0.000 & 0.000 & 0.000 \\
2 & revised & 0.012 & 0.626 & 0.006 \\
2 & unrevised & 0.000 & 0.000 & 0.000 \\
\bottomrule
\end{tabular}
\caption{Snapshot lag measured in updates to the queried slot (1{,}200 queries, keyed PFM).}
\label{tab:lag}
\end{table}

\begin{table}[h]
\centering
\scriptsize
\setlength{\tabcolsep}{3pt}
\begin{tabular}{@{}lrrrrrrrrrr@{}}
\toprule
GC & $\sigma$ p50 & p99 & max & read p50 & p99 & max & fresh p50 & p99 & max & fallback \\
\midrule
on & 2.14 & 28.5 & 34.8 & 0.00 & 0.01 & 0.03 & 7.14 & 33.5 & 39.8 & 0.192 \\
off & 1.85 & 10.6 & 16.0 & 0.00 & 0.01 & 0.01 & 6.85 & 15.6 & 21.0 & 0.117 \\
\bottomrule
\end{tabular}
\caption{Memory-induced waiting under concurrent construction (milliseconds; fresh waiting budget $\omega=5$~ms). $\sigma$ is the scheduling delay the bound of Proposition~\ref{prop:waiting} requires, computed per request as the observed fresh waiting minus $\omega$ and the reference read, not as the wake-up delay of a sleep taken elsewhere. ``read'' is the snapshot reference read $r$. Fallback: fraction of fresh requests answered from the snapshot.}
\label{tab:load}
\end{table}
\section{Store-Size Scaling}
\label{app:scaling}

Synthetic facts name two entities drawn from a pool equal to one tenth of the store size (at least 200), together with six verbs, six objects, and 40 participants, so posting lists for shared terms grow with the store. The distribution is deliberately unfavorable for term-at-a-time scoring. Each row builds a fresh store, times 300 queries after 50 warm-up queries, repeats the queries after \texttt{gc.freeze()}, and measures heap size in a separate build under \texttt{tracemalloc}. Per-observation construction cost stays near 0.14~ms across the grid. Retrieval cost follows the longest posting list traversed rather than fact count alone: with long-lived objects frozen, the median is 0.20~ms at 1{,}000 facts, 2.8~ms at 20{,}000, and 22.6~ms at 100{,}000. Without freezing, collector pauses raise the tail from 3.4~ms to 21~ms at 20{,}000 facts and from 38~ms to 159~ms at 100{,}000. The result is highly sensitive to shared vocabulary: when one participant token appears in every fact, every query traverses a posting list of length $n$ and retrieval at 20{,}000 facts is about five times slower. We therefore report posting-list structure alongside fact count.

\begin{table}[h]
\centering
\small
\begin{tabular}{@{}rrrrrrrr@{}}
\toprule
Facts & Queries & p50 & p99 & p50 (frozen) & p99 (frozen) & Add p50 & Heap (MB) \\
\midrule
100 & 300 & 0.03 & 0.04 & 0.03 & 0.04 & 0.025 & 0.4 \\
1{,}000 & 300 & 0.20 & 0.23 & 0.20 & 0.22 & 0.076 & 2.9 \\
5{,}000 & 300 & 0.70 & 0.77 & 0.70 & 0.75 & 0.132 & 13.4 \\
10{,}000 & 300 & 1.39 & 1.59 & 1.30 & 1.39 & 0.136 & 26.5 \\
20{,}000 & 300 & 2.92 & 21.3 & 2.77 & 3.37 & 0.143 & 52.9 \\
50{,}000 & 300 & 9.20 & 62.9 & 9.71 & 11.6 & 0.141 & 137.5 \\
100{,}000 & 300 & 23.0 & 158.9 & 22.6 & 37.6 & 0.142 & 275.0 \\
\bottomrule
\end{tabular}
\caption{Fresh retrieval latency (ms) by store size on the synthetic distribution; ``frozen'' rows exclude long-lived objects from garbage collection. Add: per-observation construction cost. Heap: traced allocation after the build.}
\label{tab:scaling}
\end{table}
\section{Where Lexical Retrieval Stops Being Cheaper}
\label{app:crossover}

PFM's serving cost grows with the posting lists a query traverses, while a dense retriever pays a roughly constant query-encoding cost plus a scan. Figure~\ref{fig:crossover} times both on the same synthetic stores under two term distributions: facts drawn from six verbs and six objects (wide vocabulary), and facts that all share one verb and one object (narrow vocabulary), with the entity pool at one tenth of the store in both. The distributions differ in exactly the quantity that matters: at 20{,}000 facts the longest posting list a query traverses is 3{,}327 terms in the wide setting and 20{,}000 in the narrow one.

Under the narrow distribution PFM passes MiniLM (encoding plus an exact scan, no approximate index) at about 20{,}000 facts, reaching 69~ms at 50{,}000 against dense retrieval's 20~ms. Under the wide distribution PFM is still the cheaper path at 50{,}000 facts, 15.7~ms against 17.4~ms, and the two curves meet around there. A personal store of a few thousand facts is far below either crossing, which is the regime the main text measures; a store that keeps everything a user reads, or a corpus whose facts share most of their words, is not.

\begin{figure}[h]
\centering
\includegraphics[width=0.59\columnwidth]{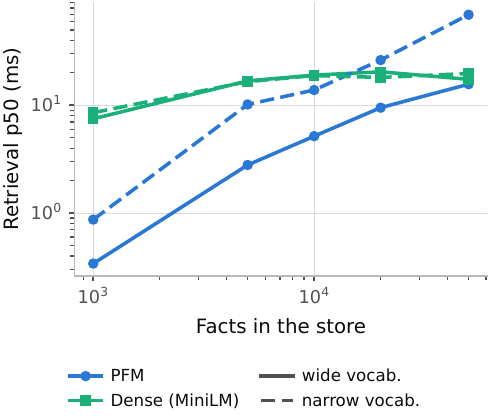}
\caption{Median retrieval latency of PFM and a MiniLM dense retriever on the same synthetic stores (100 timed queries per point after 20 warm-up queries, long-lived objects frozen). Colour is the method and the dash pattern is the term distribution. Both axes are logarithmic.}
\label{fig:crossover}
\end{figure}
\section{Additional Results}
\label{app:full}

\begin{table}[h]
\centering
\small
\setlength{\tabcolsep}{4pt}
\begin{tabular}{@{}lrrrrr@{}}
\toprule
Construction & Key prec. $\uparrow$ & Key rec. $\uparrow$ & Clean $\uparrow$ & Lost $\downarrow$ & $\Delta$ clean vs rule $\uparrow$ \\
\midrule
rule keys & 1.000 & 0.807 & \cellcolor{cPos!60}0.781 & \cellcolor{cPos!60}0.047 & -- \\
\quad model keys, llama3.2:3b & 0.930 & 0.994 & 0.592 & 0.331 & $-$0.189 {\scriptsize($-$0.241, $-$0.136)} \\
\quad model keys, qwen2:7b & 0.926 & 0.991 & 0.560 & 0.342 & $-$0.221 {\scriptsize($-$0.273, $-$0.168)} \\
\quad model keys, qwen2.5:7b & 0.990 & 0.996 & 0.644 & 0.275 & $-$0.136 {\scriptsize($-$0.186, $-$0.086)} \\
\quad model keys, qwen2.5:14b & 0.969 & 0.980 & 0.593 & 0.307 & $-$0.188 {\scriptsize($-$0.240, $-$0.134)} \\
\midrule
model updates, llama3.2:3b$^\dagger$ & -- & -- & 0.292 & 0.517 & $-$0.471 {\scriptsize($-$0.563, $-$0.379)} \\
model updates, qwen2:7b$^\dagger$ & -- & -- & 0.358 & 0.308 & $-$0.404 {\scriptsize($-$0.492, $-$0.318)} \\
model updates, qwen2.5:7b$^\dagger$ & -- & -- & 0.429 & 0.521 & $-$0.333 {\scriptsize($-$0.427, $-$0.241)} \\
model updates, qwen2.5:14b$^\dagger$ & -- & -- & 0.596 & 0.279 & $-$0.167 {\scriptsize($-$0.246, $-$0.094)} \\
keyless & -- & -- & 0.243 & 0.047 & $-$0.537 {\scriptsize($-$0.590, $-$0.488)} \\
\bottomrule
\end{tabular}
\caption{Key-assignment accuracy and prompt outcomes on 720 rewritten-corpus queries from seeds 7, 11, and 13. One rewriter (qwen2:7b) produces the text for every arm, so the key-assignment arms receive identical messages. \emph{Key prec.} and \emph{key rec.} compare assigned keys with the generator's (entity, slot) keys, counting a misfiled key as a false positive and a missing key as a false negative; Table~\ref{tab:key_errors} reports the corresponding counts. \emph{Clean} follows Table~\ref{tab:temporal}. \emph{Lost} is the share of queries whose current value appears in no active fact after a false merge. $\Delta$ is the paired difference in clean retrieval relative to the rule extractor, with a chain-cluster bootstrap 95\% interval. The Mem0-shaped path ($^\dagger$) runs at seed 7 only, so its interval resamples that seed's chains.}
\label{tab:extractor}
\end{table}

\begin{table}[h]
\centering
\small
\setlength{\tabcolsep}{4pt}
\begin{tabular}{@{}lrrrrl@{}}
\toprule
Key assigner & Correct $\uparrow$ & Missed $\downarrow$ & Wrong slot $\downarrow$ & Wrong entity $\downarrow$ & Commonest confusion \\
\midrule
rule keys & 1{,}023 & 245 & \cellcolor{cPos!60}0 & 0 & -- \\
llama3.2:3b & 1{,}173 & 7 & 88 & 0 & room $\to$ meeting (76) \\
qwen2:7b & 1{,}164 & 11 & 93 & 0 & room $\to$ meeting (74) \\
qwen2.5:7b & 1{,}250 & 5 & 13 & 0 & deadline $\to$ meeting (12) \\
qwen2.5:14b & 1{,}205 & 25 & 38 & 0 & room $\to$ meeting (18) \\
\bottomrule
\end{tabular}
\caption{Key-assignment error counts over the same 1{,}268 chain messages, pooled across three seeds. Each row partitions messages into the correct key, no key, the correct entity with the wrong slot, or the wrong entity. A missing key leaves revisions separately active; a wrong-slot key can merge distinct slots and remove a current value. No evaluated model assigns a key to the wrong entity, and three of the four models share room $\to$ meeting as their most common confusion.}
\label{tab:key_errors}
\end{table}

\begin{figure}[h]
\centering
\includegraphics[width=0.59\columnwidth]{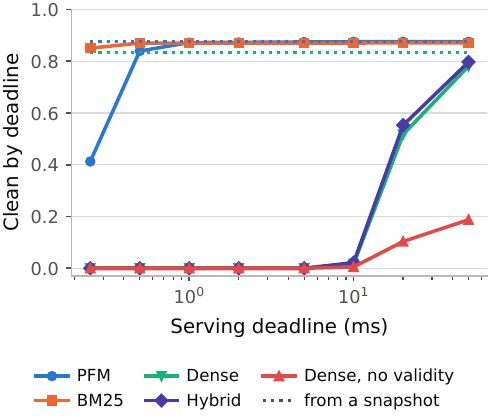}
\caption{Clean retrieval by deadline, $Y(D)$, on seed 7. Every arm receives $P_t$ and serves only active facts except the one marked otherwise. Solid lines: fresh retrieval, counting a run as successful only if it is clean and finishes within $D$; the four validity-restricted arms coincide above 1~ms, and Dense sits under Hybrid. Dotted: the same ranking served from a snapshot built after the latest revision, whose read takes about 1~$\mu$s; Appendix~\ref{app:serving} shows what happens when the snapshot is older.}
\label{fig:deadline}
\end{figure}

\begin{table}[h]
\centering
\scriptsize
\setlength{\tabcolsep}{3pt}
\begin{tabular}{@{}lrrrrrr@{}}
\toprule
Method & Clean & Current & Stale & Wrong & p50 (ms) & p99 (ms) \\
\midrule
PFM & 0.877 {\scriptsize(0.850, 0.904)} & 0.998 & 0.000 & 0.121 & 0.25 & 0.69 \\
BM25 + validity & 0.774 {\scriptsize(0.737, 0.809)} & 0.891 & 0.000 & 0.120 & 0.09 & 0.26 \\
BM25 + validity ($+P_t$) & 0.873 {\scriptsize(0.847, 0.896)} & 1.000 & 0.000 & 0.128 & 0.10 & 0.29 \\
Dense + validity & 0.834 {\scriptsize(0.802, 0.866)} & 0.998 & 0.000 & 0.165 & 17.9 & 88.5 \\
Dense + validity ($+P_t$) & 0.836 {\scriptsize(0.803, 0.868)} & 1.000 & 0.000 & 0.164 & 18.3 & 93.0 \\
Hybrid + validity & 0.807 {\scriptsize(0.773, 0.842)} & 0.949 & 0.000 & 0.151 & 17.5 & 79.7 \\
Hybrid + validity ($+P_t$) & 0.843 {\scriptsize(0.813, 0.872)} & 1.000 & 0.000 & 0.157 & 18.2 & 82.4 \\
\midrule
BM25 & 0.155 {\scriptsize(0.126, 0.186)} & 0.852 & 0.715 & 0.074 & 0.15 & 0.38 \\
BM25 ($+P_t$) & 0.195 {\scriptsize(0.164, 0.227)} & 0.997 & 0.802 & 0.007 & 0.18 & 0.46 \\
BM25 + recency & 0.185 {\scriptsize(0.154, 0.215)} & 0.877 & 0.667 & 0.102 & 0.18 & 0.55 \\
BM25 + recency ($+P_t$) & 0.193 {\scriptsize(0.162, 0.226)} & 0.989 & 0.784 & 0.038 & 0.21 & 0.60 \\
Dense (MiniLM) & 0.167 {\scriptsize(0.137, 0.198)} & 0.968 & 0.800 & 0.131 & 18.5 & 86.6 \\
Dense (MiniLM) ($+P_t$) & 0.181 {\scriptsize(0.151, 0.212)} & 0.989 & 0.802 & 0.060 & 19.0 & 97.4 \\
Hybrid (RRF) & 0.165 {\scriptsize(0.136, 0.195)} & 0.889 & 0.787 & 0.098 & 18.8 & 88.5 \\
Hybrid (RRF) ($+P_t$) & 0.188 {\scriptsize(0.157, 0.219)} & 0.998 & 0.802 & 0.027 & 19.0 & 96.4 \\
\bottomrule
\end{tabular}
\caption{All retrieval arms (1{,}200 queries; latency from seed 7). Clean intervals are two-level (seed, chain) bootstrap 95\% intervals.}
\label{tab:full}
\end{table}

\begin{table}[h]
\centering
\scriptsize
\setlength{\tabcolsep}{4pt}
\begin{tabular}{@{}lrrrrrrr@{}}
\toprule
Variant & Clean & $\Delta$ & Current & Stale & Wrong & Amb. clean & Amb. wrong \\
\midrule
full & 0.877 {\scriptsize(0.850, 0.904)} & +0.000 & 0.998 & 0.000 & 0.121 & 0.341 & 0.650 \\
no participant field & 0.842 {\scriptsize(0.813, 0.870)} & -0.036 & 0.996 & 0.000 & 0.154 & 0.170 & 0.830 \\
no participant terms & 0.811 {\scriptsize(0.775, 0.845)} & -0.067 & 0.943 & 0.000 & 0.133 & 0.211 & 0.717 \\
no participant seeds & 0.877 {\scriptsize(0.850, 0.904)} & +0.000 & 0.998 & 0.000 & 0.121 & 0.341 & 0.650 \\
no participants at all & 0.756 {\scriptsize(0.723, 0.788)} & -0.122 & 0.890 & 0.000 & 0.138 & 0.130 & 0.744 \\
no entity field & 0.884 {\scriptsize(0.857, 0.910)} & +0.007 & 0.998 & 0.000 & 0.114 & 0.377 & 0.614 \\
no graph & 0.877 {\scriptsize(0.850, 0.904)} & +0.000 & 0.998 & 0.000 & 0.121 & 0.341 & 0.650 \\
no usage & 0.877 {\scriptsize(0.850, 0.904)} & +0.000 & 0.998 & 0.000 & 0.121 & 0.341 & 0.650 \\
no recency & 0.879 {\scriptsize(0.854, 0.903)} & +0.002 & 1.000 & 0.000 & 0.121 & 0.350 & 0.650 \\
no positional decay & 0.871 {\scriptsize(0.843, 0.899)} & -0.007 & 0.998 & 0.000 & 0.128 & 0.305 & 0.686 \\
no term cap & 0.877 {\scriptsize(0.850, 0.904)} & +0.000 & 0.998 & 0.000 & 0.121 & 0.341 & 0.650 \\
no usage inheritance & 0.877 {\scriptsize(0.850, 0.904)} & +0.000 & 0.998 & 0.000 & 0.121 & 0.341 & 0.650 \\
text only dedup & 0.877 {\scriptsize(0.850, 0.904)} & +0.000 & 0.998 & 0.000 & 0.121 & 0.341 & 0.650 \\
no keys & 0.247 {\scriptsize(0.210, 0.284)} & -0.631 & 0.950 & 0.703 & 0.004 & 0.247 & 0.022 \\
pfm | participant filter & 0.998 {\scriptsize(0.993, 1.000)} & +0.121 & 0.998 & 0.000 & 0.000 & 0.991 & 0.000 \\
bm25 validity+P | participant filter & 1.000 {\scriptsize(0.995, 1.000)} & +0.123 & 1.000 & 0.000 & 0.000 & 1.000 & 0.000 \\
pfm | cutoff@0.5 & 0.898 {\scriptsize(0.873, 0.921)} & +0.021 & 0.998 & 0.000 & 0.100 & 0.453 & 0.538 \\
pfm | cutoff@0.7 & 0.983 {\scriptsize(0.973, 0.993)} & +0.106 & 0.992 & 0.000 & 0.008 & 0.933 & 0.045 \\
pfm | cutoff@0.9 & 0.866 {\scriptsize(0.835, 0.893)} & -0.012 & 0.866 & 0.000 & 0.000 & 0.839 & 0.000 \\
\bottomrule
\end{tabular}
\caption{Leave-one-out ablation and identity baselines (1{,}200 queries; stale exposure is zero in every keyed row). $\Delta$ is the paired difference from full PFM with a chain-bootstrap interval. $^\dagger$Each removed separately ($\gamma=0$, $\eta=0$, no term cap, no participant seeds, and also no usage inheritance or text-only deduplication); none changes the outcome of any of the 1{,}200 queries. Relative cutoff drops facts scoring below the stated fraction of the top score.}
\label{tab:ablation}
\end{table}

\begin{table}[h]
\centering
\small
\begin{tabular}{@{}lrrrrrrr@{}}
\toprule
Method & $d{=}0$ & $d{=}1$ & $d{=}2$ & $d{=}3$ & $d{=}4$ & Benchmark mix & Reweighted \\
\midrule
PFM (keyed) & 0.878 & 0.867 & 0.874 & 0.887 & 0.902 & 0.877 & 0.877 \\
PFM, keyless & 0.794 & 0.233 & 0.047 & 0.033 & 0.023 & 0.247 & 0.603 \\
BM25 + validity ($+P_t$) & 0.857 & 0.878 & 0.868 & 0.861 & 0.910 & 0.873 & 0.863 \\
Dense + validity ($+P_t$) & 0.819 & 0.850 & 0.818 & 0.848 & 0.857 & 0.836 & 0.826 \\
BM25 ($+P_t$) & 0.983 & 0.000 & 0.000 & 0.000 & 0.000 & 0.195 & 0.688 \\
\bottomrule
\end{tabular}
\caption{Clean retrieval by revision depth $d$, and aggregates under the benchmark's own mix of depths and under a store where 70\% of facts are never revised ($d$ weights 0.70, 0.18, 0.07, 0.03, 0.02).}
\label{tab:depth}
\end{table}

\begin{table}[h]
\centering
\small
\begin{tabular}{@{}lrrrrrr@{}}
\toprule
Method & seed 7 & seed 11 & seed 13 & seed 17 & seed 19 & Mean $\pm$ sd \\
\midrule
PFM & 0.875 & 0.912 & 0.875 & 0.854 & 0.871 & 0.877 $\pm$ 0.021 \\
BM25 + validity & 0.754 & 0.812 & 0.754 & 0.783 & 0.767 & 0.774 $\pm$ 0.025 \\
BM25 + validity ($+P_t$) & 0.871 & 0.892 & 0.867 & 0.867 & 0.867 & 0.873 $\pm$ 0.011 \\
Dense + validity ($+P_t$) & 0.833 & 0.867 & 0.812 & 0.850 & 0.817 & 0.836 $\pm$ 0.023 \\
Hybrid + validity ($+P_t$) & 0.842 & 0.863 & 0.829 & 0.858 & 0.825 & 0.843 $\pm$ 0.017 \\
BM25 ($+P_t$) & 0.212 & 0.179 & 0.208 & 0.183 & 0.192 & 0.195 $\pm$ 0.015 \\
Dense (MiniLM) ($+P_t$) & 0.200 & 0.167 & 0.179 & 0.171 & 0.188 & 0.181 $\pm$ 0.013 \\
\bottomrule
\end{tabular}
\caption{Clean retrieval per benchmark seed (240 queries each).}
\label{tab:seeds}
\end{table}

\begin{figure}[h]
\centering
\includegraphics[width=0.94\columnwidth]{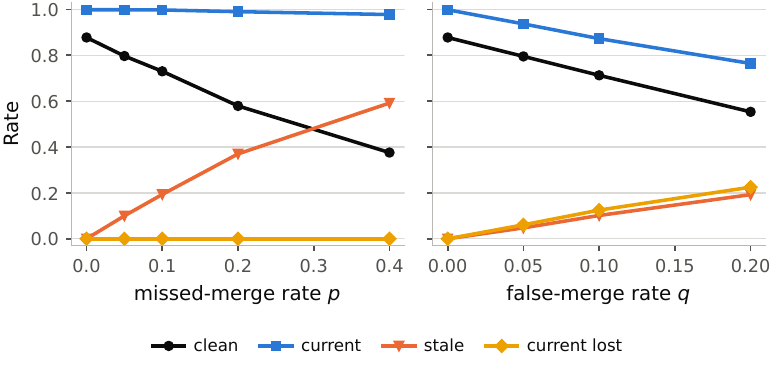}
\caption{Keyed PFM under injected key errors (1{,}200 queries). Left: missed merges at rate $p$ ($q=0$). Right: false merges at rate $q$ ($p=0$). ``Current lost'': the current value is held by no active fact, which is the failure the prompt does not show. Table~\ref{tab:keynoise} gives the full $p\times q$ grid.}
\label{fig:keynoise}
\end{figure}

\begin{table}[h]
\centering
\scriptsize
\begin{tabular}{@{}lcccc@{}}
\toprule
$p$ (rows), $q$ (columns) & 0 & 0.05 & 0.1 & 0.2 \\
\midrule
0 & 0.877 / 0.000 / 0.000 & 0.795 / 0.048 / 0.061 & 0.713 / 0.102 / 0.125 & 0.553 / 0.193 / 0.225 \\
0.05 & 0.797 / 0.100 / 0.000 & 0.709 / 0.152 / 0.068 & 0.614 / 0.210 / 0.129 & 0.496 / 0.271 / 0.214 \\
0.1 & 0.730 / 0.193 / 0.000 & 0.637 / 0.242 / 0.062 & 0.562 / 0.281 / 0.118 & 0.460 / 0.332 / 0.205 \\
0.2 & 0.579 / 0.370 / 0.000 & 0.522 / 0.398 / 0.048 & 0.462 / 0.425 / 0.102 & 0.368 / 0.463 / 0.186 \\
0.4 & 0.376 / 0.592 / 0.000 & 0.341 / 0.592 / 0.044 & 0.331 / 0.587 / 0.083 & 0.288 / 0.598 / 0.121 \\
\bottomrule
\end{tabular}
\caption{Key-noise grid: clean retrieval / stale exposure / share of queries whose current value is not active (1{,}200 queries per cell).}
\label{tab:keynoise}
\end{table}

\section{Misattribution and Abstention on LoCoMo}
\label{app:locomo}

\paragraph{Misattribution on LoCoMo.} In one store built from all ten conversations (13{,}435 facts and 537 questions, including 186 from the three conversations with a name-sharing speaker), 0.543 of prompts for a question about such a speaker contain a fact from another speaker's conversation, judged by provenance rather than by value labels. Filtering on the interlocutor reduces this exposure to 0.301 without lowering recall on LoCoMo, where the evaluated questions concern the speakers themselves. The graph affinity of Section~\ref{sec:identity} lowers the share of off-conversation lines from 0.371 to 0.237 but does not reduce name-sharing exposure (Table~\ref{tab:locomo}).

The posterior in Eq.~(\ref{eq:identity}) requires distinct stored entities, as when the store contains ``Alex Chen'' and ``Alex Rivera''. LoCoMo's two Johns are both normalized to \texttt{john}, so there is no separate candidate set for the posterior to resolve and the ranking is unchanged. Graph affinity uses who each speaker is discussed with but is insufficient to recover the distinction. Name-sharing exposure therefore remains at 0.543 on this corpus, compared with 0.121 on the controlled benchmark and 0.032 after the posterior there.

\paragraph{Abstention on LoCoMo.} Rebuilding one store per conversation and querying it with its own 537 questions plus 1{,}800 questions from the other nine conversations, the ungated ranker returns a block for almost every unanswerable request and abstains on only 7.3\%. At $\tau=1.5$, the margin gate abstains on 95.0\% of unanswerable requests but also on 73.9\% of answerable requests, reducing current-value recall from 0.220 to 0.099. Increasing $\tau$ moves further along the same tradeoff. The slot gate cannot be evaluated because LoCoMo provides no slot inventory. The controlled benchmark therefore supports schema-based abstention, while this open corpus does not provide an operating point for the schema-free margin rule that preserves useful recall.

\begin{table}[h]
\centering
\small
\setlength{\tabcolsep}{4pt}
\begin{tabular}{@{}lrrrr@{}}
\toprule
Serving rule & Wrong person & \quad name-sharing & Off-conversation lines & Current \\
\midrule
none & 0.188 & 0.543 & 0.371 & 0.078 \\
participant filter & 0.104 & 0.301 & 0.087 & 0.102 \\
context affinity & 0.190 & 0.548 & 0.237 & 0.050 \\
\bottomrule
\end{tabular}
\caption{Misattribution on LoCoMo, judged by provenance over one store built from all ten conversations (13{,}435 facts, 537 questions, of which 186 come from the three conversations with a name-sharing speaker). ``Wrong person'': the block holds a fact from a conversation whose speakers share the queried first name. ``Off-conversation lines'': the share of all served lines that come from another conversation. ``Current'': the gold answer occurs in the block, which lexical retrieval over sentence-level facts achieves rarely on this corpus, so the column bounds exposure and not recall.}
\label{tab:locomo}
\end{table}
\section{Query-Time Model Baselines}
\label{app:systems}\label{sec:systems}

Each failure can also be addressed with a model call on the request path. We compare three such baselines with the corresponding rule-based alternatives on 480 main-benchmark queries over two seeds and 180 hard-slice queries (Table~\ref{tab:systems}). A \emph{temporal reranker} receives dated candidate lines from a validity-agnostic retriever and selects the values in force. An \emph{entity resolver} receives the request and the candidate people sharing its name and chooses one. An \emph{abstention judge} receives the request and memory block and decides whether the block answers the request.

The reranker recovers less than half of what write-time supersession gives and pays for it twice. It does cut stale exposure, from 0.800 to 0.381, but current-value recall falls from 0.998 to 0.454 as it discards lines it should have kept, so clean retrieval reaches 0.302 against the keyed store's 0.894; and it costs 654~ms at the median on the request path, against 0.25~ms for the whole lexical retrieval it is reranking. The abstention judge is worse: it abstains on two thirds of the unanswerable hard-slice requests, where the slot gate abstains on all of them, and abstains so freely elsewhere that clean retrieval falls from 0.894 to 0.569, for 912~ms a request.

Entity resolution is the strongest of the three model baselines. It reduces wrong-person exposure from 0.106 to 0.031 and reaches 0.931 clean retrieval, compared with 0.969 for the posterior in Eq.~(\ref{eq:identity}). Because the resolver runs only on ambiguous mentions, its median request-path cost is zero, but its 99th percentile is 1{,}967~ms. The temporal reranker and abstention judge are both slower and less accurate than the corresponding structured rules on these tasks. Where no schema-free rule succeeds, the model call also fails to remove the underlying ambiguity or abstention tradeoff.

\begin{table}[h]
\centering
\small
\setlength{\tabcolsep}{4pt}
\begin{tabular}{@{}lrrrrr@{}}
\toprule
Arm & Stale & Wrong & Clean & Abstained (no answer) & Model ms p50 \\
\midrule
bm25+P & 0.800 & 0.006 & 0.196 & 0.000 & -- \\
llm temporal rerank & 0.381 & 0.008 & 0.302 & 0.000 & 653.6 \\
pfm & 0.000 & 0.106 & 0.894 & 0.000 & -- \\
identity hard & 0.000 & 0.031 & 0.969 & 0.000 & -- \\
identity hard|abstain slot@0.5 & 0.000 & 0.031 & 0.969 & 1.000 & -- \\
llm abstain & 0.000 & 0.054 & 0.569 & 0.667 & 912.1 \\
llm entity resolve & 0.000 & 0.031 & 0.931 & 0.000 & 0.01 \\
participant filter & 0.000 & 0.000 & 1.000 & 0.000 & -- \\
\bottomrule
\end{tabular}
\caption{Query-time model calls against the serving-side rules they would replace (qwen2:7b; 480 main-benchmark queries over seeds 7 and 11, 180 hard-slice queries). Stale, wrong and clean are on the main benchmark; ``abstained'' is on the 60 unanswerable hard-slice requests. ``Model ms'' is the median wall time of the request-path model call, which is zero for the entity resolver because it fires only on an ambiguous mention (its 99th percentile is 1{,}967~ms). The reranker is applied to BM25 + $P_t$, which is the arm that needs it; every other row serves active facts.}
\label{tab:systems}
\end{table}

\section{Response Level}
\label{app:llm}

Table~\ref{tab:llm} reports the 7B model on the main benchmark and Table~\ref{tab:llm_llama} the 3B model on the same blocks. Two further studies check that those conclusions do not rest on an unbroken prompt or on greedy decoding.

\paragraph{Prompt errors under key noise and out-of-order arrival.} Under injected key noise at $p=q=0.05$, the memory block contains a superseded value on 18.3\% of requests; the 7B and 3B models reproduce one on 4.2\% and 10.0\% of requests, respectively. Prompt-level stale exposure therefore overstates stale generation in this condition, especially for the larger model. Out-of-order arrival behaves differently: the block contains a superseded value on 13.3\% of requests and the 7B model reproduces one on the same 13.3\%. The distinction is consistent with prompt composition. Key-noise prompts usually contain the current and stale values together, whereas out-of-order prompts often contain only the stale value.

On the hard-identity slice the serving rules of Section~\ref{sec:identity} carry into the text. With no rule the block holds a competitor's value on every request and the models write it on 30.6\% and 33.3\%. The identity posterior removes the exposure on the answerable requests, leaving it only on the unanswerable ones, and the completion rate does not move (0.333 and 0.306), because the requests it still exposes are exactly the ones the models answer wrongly. Adding the slot gate removes both: 100\% of the 7B model's completions and 97.2\% of the 3B model's are correct, where correct on an unanswerable request means naming no value.

\paragraph{Sampling rather than greedy decoding.} We repeat the keyed-versus-keyless comparison at temperature 0.7 with three decoding seeds, 720 completions per arm. The keyed store reaches 0.921 correct against 0.929 greedy, and the keyless store 0.692 against 0.704, so the gap is 0.229 sampled against 0.225 greedy. The comparison does not depend on the decoding rule.

\begin{table}[h]
\centering
\small
\setlength{\tabcolsep}{3.5pt}
\begin{tabular}{@{}lrrrrrrrrrr@{}}
\toprule
Memory & Curr. & Stale & Wrong & Correct & Mixed & Stale & Wrong & Other & Chars & TTFT \\
\midrule
None & -- & -- & -- & 0.000 & 0.000 & 0.004 & 0.000 & 0.996 & 0 & 325 \\
PFM & 1.000 & 0.000 & 0.125 & 0.929 & 0.000 & 0.000 & 0.013 & 0.058 & 472 & 1229 \\
PFM, keyless & 0.908 & 0.642 & 0.013 & 0.704 & 0.008 & 0.188 & 0.000 & 0.100 & 476 & 1236 \\
BM25 & 0.854 & 0.692 & 0.092 & 0.629 & 0.000 & 0.254 & 0.004 & 0.113 & 468 & 1365 \\
BM25 + latest-value instruction & 0.854 & 0.692 & 0.092 & 0.654 & 0.000 & 0.233 & 0.004 & 0.108 & 468 & 1506 \\
BM25 + validity & 0.887 & 0.000 & 0.133 & 0.842 & 0.000 & 0.000 & 0.004 & 0.154 & 465 & 1388 \\
Dense + validity & 1.000 & 0.000 & 0.171 & 0.946 & 0.000 & 0.000 & 0.000 & 0.054 & 476 & 1469 \\
\bottomrule
\end{tabular}
\caption{Response level on the 240 seed-7 queries (qwen2:7b, greedy). Prompt columns: exposure in the injected block. Completion columns: judged outcome. Chars: mean block length. TTFT: median time to first token in ms, including a full prefill. The ``None'' arm has no memory, so its zero correct rate holds by construction.}
\label{tab:llm}
\end{table}
\begin{table}[h]
\centering
\small
\setlength{\tabcolsep}{3.5pt}
\begin{tabular}{@{}lrrrrrrrrrr@{}}
\toprule
Memory & Curr. & Stale & Wrong & Correct & Mixed & Stale & Wrong & Other & Chars & TTFT \\
\midrule
None & -- & -- & -- & 0.000 & 0.000 & 0.000 & 0.000 & 1.000 & 0 & 168 \\
PFM & 1.000 & 0.000 & 0.125 & 0.912 & 0.000 & 0.000 & 0.017 & 0.071 & 472 & 544 \\
PFM, keyless & 0.908 & 0.642 & 0.013 & 0.592 & 0.000 & 0.279 & 0.000 & 0.129 & 476 & 623 \\
BM25 & 0.854 & 0.692 & 0.092 & 0.454 & 0.000 & 0.408 & 0.004 & 0.133 & 468 & 830 \\
BM25 + latest-value instruction & 0.854 & 0.692 & 0.092 & 0.504 & 0.000 & 0.321 & 0.004 & 0.171 & 468 & 648 \\
BM25 + validity & 0.887 & 0.000 & 0.133 & 0.817 & 0.000 & 0.000 & 0.004 & 0.179 & 465 & 809 \\
Dense + validity & 1.000 & 0.000 & 0.171 & 0.887 & 0.000 & 0.000 & 0.008 & 0.104 & 476 & 829 \\
\bottomrule
\end{tabular}
\caption{Response level with the second model (llama3.2:3b, 4-bit, greedy) on the same 240 queries and the same blocks as Table~\ref{tab:llm}.}
\label{tab:llm_llama}
\end{table}
\begin{table}[h]
\centering
\scriptsize
\setlength{\tabcolsep}{3.5pt}
\begin{tabular}{@{}lllrrrrrr@{}}
\toprule
Condition & Model & Memory & $n$ & Prompt stale & Prompt wrong & Correct & Stale & Wrong \\
\midrule
hard identity & llama3.2 & PFM & 36 & 0.000 & 1.000 & 0.694 & 0.000 & 0.306 \\
hard identity & llama3.2 & pfm identity & 36 & 0.000 & 0.333 & 0.694 & 0.000 & 0.306 \\
hard identity & llama3.2 & pfm identity abstain & 36 & 0.000 & 0.000 & 0.972 & 0.000 & 0.000 \\
hard identity & qwen2 & PFM & 36 & 0.000 & 1.000 & 0.667 & 0.000 & 0.333 \\
hard identity & qwen2 & pfm identity & 36 & 0.000 & 0.333 & 0.667 & 0.000 & 0.333 \\
hard identity & qwen2 & pfm identity abstain & 36 & 0.000 & 0.000 & 1.000 & 0.000 & 0.000 \\
key noise p=0.05 q=0.05 & llama3.2 & PFM & 240 & 0.183 & 0.087 & 0.754 & 0.100 & 0.017 \\
key noise p=0.05 q=0.05 & qwen2 & PFM & 240 & 0.183 & 0.087 & 0.825 & 0.042 & 0.017 \\
out of order 0.3 & llama3.2 & PFM & 240 & 0.133 & 0.108 & 0.796 & 0.125 & 0.013 \\
out of order 0.3 & qwen2 & PFM & 240 & 0.133 & 0.108 & 0.804 & 0.133 & 0.008 \\
sampling & qwen2 & PFM & 720 & 0.000 & 0.125 & 0.921 & 0.000 & 0.010 \\
sampling & qwen2 & pfm keyless & 720 & 0.642 & 0.013 & 0.692 & 0.186 & 0.000 \\
\bottomrule
\end{tabular}
\caption{Response level where the prompt is broken on purpose. Prompt columns are exposure in the injected block; the last three are the judged completion. The sampling rows are 720 completions per arm at temperature 0.7 over three decoding seeds; every other row is greedy.}
\label{tab:llm_slices}
\end{table}

\section{Benchmark Slices}
\label{app:slices}\label{sec:slices}

The results so far hold on a benchmark where each slot has one value, messages arrive in the order they were written, the participant of a message is the person the fact is about, and every query has an answer. Four slices break one assumption each, with the same generator, arms and metrics (Table~\ref{tab:slices}).

Key granularity cuts both ways: when one person holds both a standup and a sync, a key at (entity, slot) merges them and current-value recall falls to 0.500, while a key at (entity, slot, noun) recalls every current value and the keyless store exposes a superseded value on every query. Coarse keys delete and fine keys stop recognizing revisions; the main benchmark hides the choice because each (entity, slot) pair has one instance. Arrival order is not statement order: delivering 30\% of chains out of order costs a transaction-time store 0.130 clean retrieval (0.877 to 0.747) with stale exposure 0.156, because a late message restating an old value closes the newer one, and Proposition~\ref{prop:one_active} still holds with the wrong value active. Ordering by stated time restores the in-order numbers exactly, at the cost of trusting a stated time at commit.

The remaining slices concern identity and abstention. Clean retrieval remains between 0.865 and 0.965 when the participant is the subject, a third party, a group member, or missing because the fact text still names its subject. A hard participant filter, however, recalls the current value on only 2.0\% of third-party facts and 1.0\% of facts with no participant, showing the recall cost of treating the interlocutor as the only eligible identity. None of the ungated retrieval arms abstains. With same-first-name people whose messages share slot, noun, template and arrival day, every arm exposes the competitor's value on every query except the participant filter and a relative cutoff at 0.7 (which leaves 0.092); on queries whose answer is not in the store, every arm fills the block with someone else's value of the right type, and the cutoff does not help, because the exposure comes from the top-scoring facts.

\paragraph{Leave-one-out ablation.}\label{sec:ablation}

Removing one component at a time (Table~\ref{tab:ablation}) leaves participant signals as the only ranking components with a measurable effect: dropping the participant field, the participant query terms, or both lowers clean retrieval by 0.036, 0.067, and 0.122, and on ambiguous queries clean retrieval falls from 0.341 to 0.130. The entity field, graph, usage term, recency, and term cap do not change aggregate clean retrieval on this benchmark; removing the graph or usage term changes no query outcome. These null effects are specific to the benchmark and do not establish that the components are generally useless. Combined with the equivalence result in Section~\ref{sec:temporal}, the retrieval evidence supports an active-only store with participant-aware lexical retrieval. The graph has a separate role in Section~\ref{sec:idresults}, where it contributes context affinity for resolving ambiguous names rather than improving the baseline ranking score.

\paragraph{Identity baselines on the main benchmark.} A hard participant filter removes all wrong-person exposure and lifts clean retrieval to 0.998; a relative score cutoff at 0.7 reaches 0.983, and at 0.9 it removes the exposure along with 13\% of current values. Both exploit a property of this benchmark, that the participant is the subject of the fact, which Section~\ref{sec:slices} removes and Section~\ref{sec:idresults} replaces with a rule that does not need it.

\begin{table}[h]
\centering
\scriptsize
\setlength{\tabcolsep}{4pt}
\begin{tabular}{@{}lrrrrr@{}}
\toprule
Condition & $n$ & Current & Stale & Wrong & Clean \\
\midrule
\emph{Two instances of one slot type} &  &  &  &  &  \\
key(entity,slot) & 240 & 0.500 & 0.000 & 0.000 & 0.500 \\
key(entity,slot,noun) & 240 & 1.000 & 0.000 & 0.000 & 1.000 \\
keyless & 240 & 1.000 & 1.000 & 0.000 & 0.000 \\
\midrule
\emph{Out-of-order arrival} &  &  &  &  &  \\
in order & 1200 & 0.998 & 0.000 & 0.121 & 0.877 \\
out of order & 1200 & 0.846 & 0.156 & 0.103 & 0.747 \\
out of order, event time & 1200 & 0.998 & 0.000 & 0.121 & 0.877 \\
\midrule
\emph{Participant is not the subject} &  &  &  &  &  \\
pfm, subject & 400 & 0.995 & 0.000 & 0.068 & 0.927 \\
pfm, third party & 200 & 0.985 & 0.000 & 0.120 & 0.865 \\
pfm, group & 400 & 0.995 & 0.000 & 0.030 & 0.965 \\
pfm, missing & 200 & 1.000 & 0.000 & 0.070 & 0.930 \\
pfm, participant filter, subject & 400 & 1.000 & 0.000 & 0.005 & 0.995 \\
pfm, participant filter, third party & 200 & 0.020 & 0.000 & 0.005 & 0.020 \\
pfm, participant filter, group & 400 & 0.998 & 0.000 & 0.000 & 0.998 \\
pfm, participant filter, missing & 200 & 0.010 & 0.000 & 0.000 & 0.010 \\
\midrule
\emph{Hard name collisions and unanswerable queries} &  &  &  &  &  \\
pfm & 120 & 1.000 & 0.000 & 1.000 & 0.000 \\
pfm, unanswerable & 60 & 0.000 & 0.000 & 1.000 & 0.000 \\
pfm, cutoff@0.7 & 120 & 1.000 & 0.000 & 0.092 & 0.908 \\
pfm, cutoff@0.7, unanswerable & 60 & 0.000 & 0.000 & 1.000 & 0.000 \\
pfm, participant filter & 120 & 1.000 & 0.000 & 0.000 & 1.000 \\
pfm, participant filter, unanswerable & 60 & 0.000 & 0.000 & 1.000 & 0.000 \\
\bottomrule
\end{tabular}
\caption{Slices that break one assumption of the main benchmark. Rows within a block share a corpus. For unanswerable queries, ``wrong'' counts blocks that contain some other person's value of the queried slot type, and ``clean'' is the second case of Eq.~(\ref{eq:clean}), which requires exposing none; no arm in this table can return an empty block, so every one of them scores zero there, and Table~\ref{tab:identity} adds the arms that can.}
\label{tab:slices}
\end{table}
\section{Proofs}
\label{app:proofs}

\subsection{Proposition~\ref{prop:one_active}}

\begin{proof}
Order the commits that carry key $\kappa$ by commit time $u_1<u_2<\cdots$. Because commits are atomic with respect to readers, every reader-visible state lies between two consecutive commits, and it suffices to show the claim after each commit. Before $u_1$ it holds by assumption. Suppose it holds before commit $j$. If no fact with key $\kappa$ is active, commit $j$ activates the new fact $f_j$ with $a_{f_j}=u_j$ and $b_{f_j}=\infty$, leaving exactly one active fact. Otherwise the induction hypothesis says the active fact $f$ with key $\kappa$ is unique, and the commit sets $b_f=u_j$ and $a_{f_j}=u_j$. Since intervals are half-open, $f$ is inactive and $f_j$ active at every $\tau\ge u_j$ until a later commit closes $f_j$, so the count stays at one. Commits for other keys and pruning (which never removes the active fact of a key) do not change activity intervals of facts with key $\kappa$. By induction the claim holds after every commit and hence in every reader-visible state.
\end{proof}

The argument uses only the commit order. If revisions arrive out of order, the invariant still holds but the surviving value is the latest to arrive.

\subsection{Proposition~\ref{prop:waiting}}

\begin{proof}
Snapshot mode. The request thread performs one reference read. By (ii) it never blocks on construction or a worker, by (v) it runs within $\sigma$ of becoming runnable, and by (iii) the read takes at most $r$; by (i) the reference always points to a complete snapshot. The added time is at most $\sigma+r$.

Fresh mode. The request thread enqueues its context, which by (iv) takes at most $h$ and does not block on the worker, then waits on a completion event with timeout $\omega_t$. If the worker finishes within $\omega_t$, the wait returns and, by (v), the request resumes at most $\sigma$ later, for at most $h+\omega_t+\sigma$. Otherwise the timeout expires, the thread sets the cancellation flag without blocking, and reads the snapshot as in snapshot mode, for a total of at most $h+\omega_t+\sigma+r$. In neither case does the request thread wait for a lock or an operation of construction, so the bounds do not depend on extraction, index maintenance, graph updates, or pruning.
\end{proof}

The bound is informative only when $h$ and $\sigma$ are measured on the serving path, as in Appendix~\ref{app:serving}; a scheduling delay measured elsewhere in the process need not bound $\sigma$ here. A worker may also wait for the store lock, for example during pruning. Cancellation lets it exit as soon as the lock is acquired, but with a single worker a fresh request can still queue behind a cancelled retrieval that is draining, and that queueing is part of $h$. The bound excludes model latency and does not constrain the age of the snapshot returned.

\end{document}